\documentclass[12pt]{article}
\usepackage[colorlinks=true,linkcolor=blue,citecolor=Red]{hyperref}
\usepackage{breakcites}

\usepackage{mathtools}

\usepackage{subfigure}

\usepackage{algorithm}
\usepackage{algpseudocode}

\usepackage{latexsym}
\usepackage{graphics}
\usepackage{amsmath}
\usepackage[algo2e]{algorithm2e}
\usepackage{xspace}
\usepackage{amssymb}
\usepackage{psfrag}
\usepackage{epsfig}
\usepackage{amsthm}
\usepackage{url}
\usepackage{pst-all}

\usepackage{color}

\definecolor{Red}{rgb}{1,0,0}
\definecolor{Blue}{rgb}{0,0,1}
\definecolor{Olive}{rgb}{0.41,0.55,0.13}
\definecolor{Yarok}{rgb}{0,0.5,0}
\definecolor{Green}{rgb}{0,1,0}
\definecolor{MGreen}{rgb}{0,0.8,0}
\definecolor{DGreen}{rgb}{0,0.55,0}
\definecolor{Yellow}{rgb}{1,1,0}
\definecolor{Cyan}{rgb}{0,1,1}
\definecolor{Magenta}{rgb}{1,0,1}
\definecolor{Orange}{rgb}{1,.5,0}
\definecolor{Violet}{rgb}{.5,0,.5}
\definecolor{Purple}{rgb}{.75,0,.25}
\definecolor{Brown}{rgb}{.75,.5,.25}
\definecolor{Grey}{rgb}{.5,.5,.5}

\newcommand{\E}[1]{\mathbb{E}\left[#1\right]}

\newcommand{\R}{\mathbb{R}}

\newcommand{\ip}[2]{\langle{#1},{#2}\rangle}

\renewcommand{\E}[1]{\mathbb{E}\!\left[#1\right]}
\renewcommand{\R}{\mathbb{R}}

\newcommand{\distr}{\stackrel{d}{=}}

\newcommand{\ignore}[1]{\relax}

\newlength\myindent
\newtheorem{theorem}{Theorem}[section]

\newtheorem{lemma}[theorem]{Lemma}

\newtheorem{proposition}[theorem]{Proposition}
\newtheorem{coro}[theorem]{Corollary}
\newtheorem{claim}[theorem]{Claim}

\newcommand{\risk}[1]{\widehat{\mathcal{L}}(#1)}

\newcommand{\riskk}{\widehat{\mathcal{L}}}


\newcounter{parentnumber}

\makeatletter
\def\BState{\State\hskip-\ALG@thistlm}
\makeatother

\definecolor{Red}{rgb}{1,0,0}
\definecolor{Blue}{rgb}{0,0,1}
\definecolor{Olive}{rgb}{0.41,0.55,0.13}
\definecolor{Green}{rgb}{0,1,0}
\definecolor{MGreen}{rgb}{0,0.8,0}
\definecolor{DGreen}{rgb}{0,0.55,0}
\definecolor{Yellow}{rgb}{1,1,0}
\definecolor{Cyan}{rgb}{0,1,1}
\definecolor{Magenta}{rgb}{1,0,1}
\definecolor{Orange}{rgb}{1,.5,0}
\definecolor{Violet}{rgb}{.5,0,.5}
\definecolor{Purple}{rgb}{.75,0,.25}
\definecolor{Brown}{rgb}{.75,.5,.25}
\definecolor{Grey}{rgb}{.5,.5,.5}
\definecolor{Pink}{rgb}{1,0,1}
\definecolor{DBrown}{rgb}{.5,.34,.16}
\definecolor{Black}{rgb}{0,0,0}

\usepackage{float}

\usepackage{float}
\author{
{\sf David Gamarnik}\thanks{MIT; e-mail: {\tt gamarnik@mit.edu}. Research supported  by the NSF grants CMMI-1335155.}
\and
{\sf Eren C. K{\i}z{\i}lda\u{g}\thanks{MIT; e-mail: {\tt kizildag@mit.edu}}}
\and
{\sf Ilias Zadik}\thanks{NYU; e-mail: {\tt zadik@nyu.edu}. Research supported by a CDS Moore-Sloan Postdoctoral Fellowship.}
}

\begin{document}

\title{Stationary Points of Shallow Neural Networks with Quadratic Activation Function}
\date{}

\maketitle
\begin{abstract}
        We consider the teacher-student setting of learning shallow neural networks with quadratic activations and planted weight matrix $W^*\in\R^{m\times d}$, where $m$ is the width of the hidden layer and $d\leqslant m$ is the dimension of data. We study the optimization landscape associated with the empirical and the population squared risk of the problem. Under the assumption the planted weights are full-rank we obtain the following results. 
        
        First, we establish that the landscape of the empirical risk $\risk{W}$ admits an "energy barrier" separating rank-deficient $W$ from $W^*$: if $W$ is rank deficient, then $\risk{W}$ is bounded away from zero by an amount we quantify. We then couple this result by showing that, assuming number $N$ of samples grows at least like a polynomial function of $d$, all full-rank approximate stationary points of $\risk{W}$ are nearly global optimum. These two results allow us to prove that gradient descent, when initialized below the energy barrier, approximately minimizes the empirical risk and recovers the planted weights in polynomial-time. 
        
        Next, we show that initializing below the aforementioned energy barrier is in fact easily achieved when the weights are randomly generated under relatively weak assumptions. We show that provided the network is sufficiently overparametrized, initializing with an appropriate multiple of the identity suffices to obtain a risk below the energy barrier. At a technical level, the last result is a consequence of the semicircle law for the Wishart ensemble and could be of independent interest. 
        
        Finally, we study the minimizers of the empirical risk and identify a simple necessary and sufficient geometric condition on the training data under which any minimizer has necessarily zero generalization error. We show that as soon as $N\geqslant N^*=d(d+1)/2$, randomly generated data enjoys this geometric condition almost surely, while if $N<N^*$, that ceases to be true.
\end{abstract}

\newpage
\tableofcontents
\newpage
\section{Introduction}
Neural network architectures are demonstrated to be extremely powerful in practical tasks such as natural language processing \cite{collobert2008unified}, image recognition \cite{he2016deep}, image classification \cite{krizhevsky2012imagenet}, speech recognition \cite{mohamed2011acoustic}, and game playing \cite{silver2017mastering}; and is becoming popular in other areas, such as applied mathematics \cite{chen2018neural,weinan2017deep}, clinical diagnosis \cite{de2018clinically}; and so on. Despite this empirical success, a mathematical understanding of these architectures is still largely missing. 

While it is NP-hard to train such architectures, it has been observed empirically that the gradient descent, albeit being a simple first-order local procedure, is rather successful in training such networks. This is somewhat surprising due to the highly non-convex nature of the associated objective function. Our main motivation in this paper is to provide further insights into the optimization landscape and generalization abilities of these networks.

\subsection{Model, Contributions, and Comparison with the Prior Work}
\paragraph{Model.} In this paper, we consider a shallow neural network architecture with one hidden layer of width $m$ (namely, the network consists of $m$ neurons). We study it under the realizable model assumption, that is, the labels are  generated by a teacher network with ground truth weight matrix $W^*\in\R^{m\times d}$ whose $j^{\rm th}$  row $W_j^*$ carries the weights of $j^{\rm  th}$ neuron. We also assume that the data we observe are drawn using input data $X\in\R^d$ with centered i.i.d. sub-Gaussian coordinates.  Note that such shallow architectures with planted weights and Gaussian input data have been explored extensively in the literature, see e.g. \cite{du2017gradient,li2017convergence,tian2017analytical,zhong2017recovery,soltanolkotabi2017learning,brutzkus2017globally}.

 Our focus is in particular on networks with quadratic activation, studied also by Soltanolkotabi et al. \cite{soltanolkotabi2018theoretical}; and Du and Lee \cite{du2018power}, among others. 
This object, an instance of what is known as a \emph{polynomial network} \cite{livni2014computational}, computes for every input data $X\in\mathbb{R}^d$ the function:
\begin{equation}\label{eq:fnc-nn-computes}
f(W^*;X)=\sum_{j=1}^m \ip{W^*_j}{X}^2 =\|W^*X\|_2^2.
\end{equation} 
We note that albeit being a stylized activation function, blocks of quadratic activations can be stacked together to approximate deeper networks with sigmoid activations as shown by Livni et al. \cite{livni2014computational}; and furthermore this activation serves as a second order approximation of general nonlinear activations as noted by Venturi et al. \cite{venturi2018spurious}. Thus, we study the quadratic networks as an attempt to gain further insights on more complex networks.

Let $X_i\in\R^d$, $1\leqslant i\leqslant N$ be a (i.i.d.) collection of input data, and let $Y_i=f(W^*;X_i)$ be the corresponding label generated per \eqref{eq:fnc-nn-computes}. The goal of the learner is as follows: given the training data $(X_i,Y_i)\in\R^d\times \R$, $1\leqslant i\leqslant N$, find a weight matrix $W\in\R^{m\times d}$ that explains the input-output relationship on the training data set in the best possible way, often by solving the so-called ``\emph{empirical risk minimization}" (ERM) optimization problem
\begin{align}\label{eq:empirical-risk}
\min_{W\in\R^{m\times d}}  \risk{W}
\quad\text{where}\quad
\risk{W}\triangleq  \frac1N\sum_{1\leqslant i\leqslant N} (Y_i-f(W;X_i))^2;
\end{align}
and understand its generalization ability, quantified by the ``\emph{generalization error}" (also known as the ``\emph{population risk}" associated with any solution candidate $W\in \R^{m\times d}$) that is given by
$\mathcal{L}(W) \triangleq \mathbb{E}[(f(W^*;X) - f(W;X))^2]$,
where the expectation is with respect to a ''fresh'' sample $X$, which has the same distribution as $X_i, 1\leqslant i\leqslant N$, but is independent from the sample. The landscape of the loss function $\risk{\cdot}$ is non-convex, therefore rendering the optimization problem difficult. Nevertheless, the gradient descent algorithm, despite being a simple first-order procedure, is rather successful in training neural nets in general: it appears to find a $W\in\R^{m\times d}$ with near-optimal $\risk{W}.$   Our partial motivation is to investigate this phenomenon in the case the activation is quadratic.

\paragraph{Contributions.} Despite working on a stylized model, our work provides a series of results for plenty aspects of the training and generalization abilities of such networks, hopefully bringing insights for more complex networks. We provide multiple results pertaining both the empirical and the population risks. The results for the latter require a milder distributional assumption that it suffices for the data $X_i$ to have centered i.i.d. coordinates  with finite fourth moment; and are provided under the supplementary material due to space constraints.

We first study the landscape of risk functions and quantify an \emph{``energy barrier"} separating rank-deficient matrices from the planted weights. Specifically, if $W^*\in\R^{m\times d}$ is full-rank, then the risk function for any rank-deficient $W$ is bounded away from zero by an explicit constant - independent of $d$- controlled by the smallest singular value $\sigma_{\min}(W^*)$ as well as the second and the fourth moments of the data. See Theorem \ref{thm:energy-barrier-empirical} for the empirical, and Theorem~\ref{thm:band-gap} for the population version.

We next study the \emph{full-rank stationary points} of the risk functions and the \emph{gradient descent} performance. 
We first establish that when $W^*$ is full rank, any full-rank stationary point $W$ of the risk functions is necessarily global minimum, and that any such $W$ is of form $W=QW^*$ where $Q\in\R^{m\times m}$ is orthonormal.  
See Theorem \ref{thm:full-rank-empirical-global-opt} for the empirical; and Theorem \ref{thm:full-rank-global-opt} for the population version. 
We then establish that all full-rank ``approximate" stationary points $W$ of $\risk{\cdot}$ below the aforementioned ``energy barrier", are ``nearly" global optimum. Furthermore, we establish that if the number $N$ of samples is $\mathrm{poly}(d)$, then the weights $W$ of any full-rank ``approximate" stationary point are uniformly close to $W^*$. As a corollary, gradient descent with initialization below the ``energy barrier" in time  ${\rm poly}(\frac1\epsilon,d)$ recovers a solution $W$ for which the weights are $\epsilon$-close to planted weights, and thus the generalization error $\mathcal{L}(W)$ is at most $\epsilon$. The bound on $\mathcal{L}(W)$ is derived by controlling the condition number of a certain  matrix whose i.i.d. rows consists of tensorized data $X_i^{\otimes 2}$; using a recently developed machinery \cite{emschwiller2020neural} studying the spectrum of expected covariance matrices of tensorized data. 
See Theorem~\ref{thm:gd-conv-empirical} for the empirical; and Theorem~\ref{thm:gd-conv} for the population version.

Next, we study the question of whether one can \emph{initializate} below the aforementioned energy barrier. We answer affirmatively this question in the context of randomly generated $W^*\in\R^{m\times d}$, and establish in Theorem~\ref{thm:initialization-empirical} that as long as the network is sufficiently overparametrized, specifically $m>Cd^2$, it is possible to initialize $W_0$ such that w.h.p. the risk associated to  $W_0$ is below the required threshold. This is achieved using random matrix theory, specifically a semicircle law for Wishart matrices which shows the spectrum of $(W^*)^TW^*$ is tightly concentrated \cite{bai1988convergence}. 
It is also worth noting that networks with random weights are an active area of research: they play an important role in the analysis of complex networks by providing further insights; 
define the initial loss landscape; and also are closely related to random feature methods: Rahimi \& Recht \cite{rahimi2009weighted} showed that shallow architectures trained by choosing the internal weights randomly and optimizing only over the output weights return a classifier with reasonable generalization performance at accelerated training speed. Random shallow networks were also shown to well-approximate dynamical systems  \cite{gonon2020approximation}; have been successfully employed in the context of extreme learning machines \cite{huang2006extreme}; and were studied in the context of random matrix theory, see \cite{pennington2017nonlinear} and references therein.



Our next focus is on the \emph{sample complexity} for generalization. While we study the landscape of the empirical risk, it is not by any means certain that the optimizers of
$\min_W\risk{W}$ also achieve zero generalization error.  We give necessary and sufficient conditions 
on the samples $X_i,1\leqslant i\leqslant  N$ so that any minimizer had indeed zero generalization error in our setting. We show that, if ${\rm span}(X_iX_i^T:i\in[N])$ is the space of all $d\times d$-dimensional real symmetric matrices, then any global minimum of the empirical risk is necessarily a global optimizer of the population risk, and thus, has zero generalization error. Note that, this geometric  condition is not {\em retrospective} in manner: it can be checked ahead of the optimization task by computing ${\rm span}(X_iX_i^T:i\in[N])$. Conversely, we show that if the span condition above is not met then there exists a global minimum $W$ of the empirical risk function
which induces a strictly positive generalization error. This is established in Theorem~\ref{thm:geo-condition}. 

To complement our analysis, we then ask the following question: what is the ``\emph{critical number}" $N^*$ of the training samples, under which 
the (random) data $X_i, 1\leqslant i\leqslant N$ enjoys the aforementioned  span condition? We prove this number to be $N^*=d(d+1)/2$, 
under a very mild assumption that the coordinates of $X_i\in\R^d$ are jointly continuous. This is shown in Theorem~\ref{thm:random-data-geo-cond}.
Finally, in Theorem~\ref{thm:main} we show that when $N<N^*$ not only there exists $W$ with zero empirical risk and strictly positive
generalization error, but we bound this error from below by an amount very similar to the bound for rank-deficient matrices discussed in our
earlier Theorem~\ref{thm:energy-barrier-empirical}. 

We end with a comment on overparametrization and generalization. A common paradigm in statistical learning theory is that, overparametrized models, that is, models with more parameters than necessary, while being capable of interpolating the training data, tend to generalize poorly because of overfitting to the proposed model. Yet, it has been observed empirically that neural networks tend to not suffer from this complication \cite{zhang2016understanding}: despite being overparametrized, they seem to have good generalization performance, provided the interpolation barrier is exceeded. In Theorem~\ref{thm:geo-condition} ${\rm (a)}$ we establish the following result which potentially sheds some light on this phenomenon for the case of shallow neural networks with quadratic activations. More concretely, we establish the following: suppose that the data enjoys the aforementioned geometric condition. Then, any interpolator achieves zero generalization error, even when the interpolator is a neural network with potentially larger number $\widehat{m}$ of internal nodes compared to the one that generated the data, namely by using a weight matrix $W\in\R^{\widehat{m}\times d}$ where $\widehat{m} \geqslant m$. In other words, the model does not overfit to the much larger width of the interpolator. 

\paragraph{Comparison with \cite{soltanolkotabi2018theoretical} and \cite{du2018power}.} We now make a comparison with two very related prior work, studying the quadratic activations. We start with the work by Soltanolkotabi,  Javanmard and  Lee \cite{soltanolkotabi2018theoretical}. In \cite[Theorem~2.2]{soltanolkotabi2018theoretical}, the authors study the empirical risk landscape of a slightly 
more general version of our model: $Y_i = \sum_{j=1}^m v_j^* \ip{W_j^*}{X_i}^2$, assuming ${\rm rank}(W^*)=d$ like us, and assuming all non-zero entries of $v^*$ have the same sign. Thus our model is the special case where all entries of $v^*$ equal unity. The authors establish that as long as $d\leqslant N\leqslant cd^2$ for some small fixed constant $c$, every local minima of the empirical risk 
function is also a global minima (namely, there exists no spurious local minima), 
and furthermore, every saddle point has a direction of negative curvature. As a result they show that gradient descent with an arbitrary initialization
converges to a globally optimum solution of the ERM problem (\ref{eq:empirical-risk}). In particular, their result does not require the initialization point to be below some risk value (the energy barrier), like in our case. Nevertheless, our results show that one needs not to worry about saddle points below the energy barrier as none exists per our Theorem~\ref{thm:energy-barrier-empirical}. 
Importantly, though,  the regime $N<cd^2$ for small $c$ that \cite[Theorem~2.2]{soltanolkotabi2018theoretical} applies  is below the \emph{provable sample complexity value}  $N^*=d(d+1)/2$ when the data are drawn from a continuous distribution as per our Theorem~\ref{thm:random-data-geo-cond}. In particular, as we establish when $N<N^*$, the ERM problem \eqref{eq:empirical-risk} admits global optimum solutions with zero empirical risk value, but with generalization error bounded away from zero. Thus, the regime $N<N^*$ does not correspond to the regime where solving the ERM has a guaranteed control on the generalization error. The same theorem in \cite{soltanolkotabi2018theoretical} also studies the approximate stationary points, and shows that for any such point $W$, $\risk{W}$ is also small. Our Theorem~\ref{thm:gd-conv-empirical}, though, takes a step further and shows that not only the empirical risk is small but the recovered $W$ is close to planted weights $W^*$; and therefore it has small generalization error $\mathcal{L}(W)$.

 It is also worth noting that albeit not being our focus in the present paper,  \cite[Theorem~2.1]{soltanolkotabi2018theoretical} also studies the landscape of the empirical risk when a quadratic network model $X\mapsto \sum_{j=1}^m v_j^* \ip{W_j^*}{X}^2$ is used for interpolating arbitrary input/label pairs $(X_i,Y_i)\in\R^d\times \R$, $1\leqslant i\leqslant N$, that is, without making an assumption that the labels are generated according to a network with planted weights. They establish similar landscape results; namely, the absence of spurious local minima, and the fact that every saddle point has a direction of negative curvature, as long as the output weights $v^*$ has at least $d$ positive and $d$ negative entries (consequently, the width $m$ has to be at least $2d$). While this result does not assume any rank condition on $W$ like us, it bypasses this technicality at the cost of assuming that the output weights contain at least $d$ positive and $d$ negative weights, and consequently, by assuming $m$ is at least $2d$, namely when the network is sufficiently wide.

 Yet another closely related work studying quadratic activations is the paper by Du and Lee \cite{du2018power}. This paper establishes that for any smooth and convex loss $\ell(\cdot,\cdot)$, the landscape of the regularized loss function $\frac1N\sum_{i=1}^N \ell(f(W;X_i),Y_i)+\frac{\lambda}{2}\|W\|_F^2$ still admits aforementioned favorable geometric characteristics. Furthermore, since the learned weights are of bounded Frobenius norm due to norm penalty $\|W\|_F^2$ imposed on objective, they retain good generalization via Rademacher complexity considerations. While this work addresses the training and generalization when the norm of $W$ is controlled during training; it does not carry out approximate stationarity analysis like Soltanolkotabi et al. \cite{soltanolkotabi2018theoretical} and we do; and does not study their associated loss/generalization like in our case. Even though they show bounded norm optimal solutions to the optimization problem with modified objective generalize well; it remains unclear from their analysis whether the approximate stationary points of this objective also have well-controlled norm. 
\paragraph{Further relevant prior work.}
As noted in the introduction, neural networks achieved remarkable empirical success
which fueled research starting from the expressive ability of these networks, going as early as Barron \cite{barron1994approximation}. More recent works along this front focused on deeper and sparser models, see e.g. \cite{mhaskar2016learning,telgarsky2016benefits,eldan2016power,schmidt2017nonparametric,poggio2017and,bolcskei2019optimal}. In particular, the expressive power of such network architectures is relatively well-understood. Another issue pertaining such architectures is computational tractability: Blum and Rivest  established in \cite{blum1989training} that it is NP-complete to train a very simple, 3-node, network; whose nodes compute a linear thresholding function. Despite this worst-case result, it has been observed empirically that local search algorithms (such as gradient descent), are rather successful in training. While several authors, including \cite{sedghi2014provable,janzamin2015beating,goel2016reliably}, devised provable training algorithms for such nets; these algorithms unfortunately are based on methods other than the gradient descent; thus not shedding any light on its apparent empirical success.

On a parallel front, many papers studied the behaviour of the GD by analyzing the trajectory of it or its stochastic variant, under certain stylistic assumptions on the data as well as the network. These assumptions include Gaussian inputs, shallow networks (with or without the convolutional structure) and the existence of planted weights (the so-called teacher network) generating the labels. Some partial and certainly very incomplete references to this end include \cite{tian2017analytical, brutzkus2017globally,brutzkus2017sgd,zhong2017learning,soltanolkotabi2017learning,li2017convergence,du2017gradient}. Later work relaxed the distributional assumptions. For instance,  \cite{du2017convolutional} studied the problem of learning a convolutional unit with ReLU with no specific distributional assumption on input, and established the convergence of SGD with rate depending on the smoothness of the input distribution and the closeness of the patches. Several other works along this line, in particular under the presence of overparametrization, are the works by Du et al. \cite{du2018gradient,du2018gradient2}.

Yet another line of research on the optimization front, rather than analyzing the trajectory of the GD, focuses on the mean-field analysis: empirical  distribution of the parameters  of network with infinitely many internal nodes can be described as a Wasserstein gradient flow, thus some tools from the theory of optimal transport can be used, see e.g. \cite{wei2018margin,rotskoff2018neural,chizat2018global,song2018mean,sirignano2019mean}. Albeit explaining the story to some extent for infinitely wide networks, it remains unclear whether these techniques provide results for a more realistic network model with finitely many internal nodes.

As noted earlier, the optimization landscape of such networks is usually highly non-convex. More recent research on such non-convex objectives showed that if the landscape has certain favorable geometric properties such as the absence of spurious local minima and the existence of direction with negative curvature for every saddle point, 
local methods can escape the saddle points and converge to the global minima. Examples of this line of research on loss functions include \cite{ge2015escaping,levy2016power,lee2016gradient,jin2017escape,du2017gradientexponent}. Motivated by this front of research, many papers analyzed geometric properties of the optimization landscape, including \cite{poston1991local, haeffele2014structured, choromanska2015loss,haeffele2015global,kawaguchi2016deep,hardt2016identity,soudry2016no,freeman2016topology,zhou2017landscape,nguyen2017loss,ge2017learning,safran2017spurious,soudry2017exponentially,zhou2017critical,nguyen2018loss,venturi2018spurious,du2018power,soltanolkotabi2018theoretical}. 

We now touch upon yet another very important focus, that is the \emph{generalization ability} of such networks: how well a solution found, e.g. by GD, predicts an unseen data? A common paradigm in statistical learning theory is that overparametrized models tend to generalize poorly. Yet, neural networks tend to not suffer from this complication \cite{zhang2016understanding}. Since the VC-dimension of these networks grow (at least) linear in the number of parameters \cite{harvey2017nearly,bartlett2019nearly}, standard Vapnik-Chervonenkis theory do not help explaining the good generalization ability under presence of overparametrization. This has been studied, among others, through the lens of the weights of the norm matrices \cite{neyshabur2015norm,bartlett2017spectrally,liang2017fisher,golowich2017size,dziugaite2017computing,wu2017towards}; PAC-Bayes theory \cite{neyshabur2017pac,neyshabur2017exploring}, and compression-based bounds \cite{arora2018stronger}. A main drawback is that these papers require some sort of constraints on the weights and are mostly \emph{a posteriori}: whether or not a good generalization takes place can be determined only when the training process is finished. A recent work by Arora et al. \cite{arora2019fine} provided an \emph{a priori} guarantee for the solution found by the GD.


\paragraph{Paper organization.} In Section \ref{sec:landscape} we present our main results on the landscape of the risk functions, including our energy barrier result for rank-deficient matrices, our result about the absence of full-rank stationary points of the risk function
except the globally optimum points; and our result on the convergence of gradient descent.
In Section \ref{sec:random-initialization}, we present our results regarding randomly generated weight matrices $W^*$ and sufficient conditions for
good initializations. In Section \ref{sec:data-no}, we study the critical number of training samples guaranteeing good generalization property. We collect useful auxiliary lemmas in Section~\ref{sec:auxi}; and provide
the proofs of all of our results in Section~\ref{sec:proofs}. 

\paragraph{Notation.}
The set of reals,  positive reals; and the set $\{1,2,\dots,k\}$ are denoted by $\mathbb{R}$, $\R_+$, and $[k]$. For any matrix $A$, its smallest and largest singular values, spectrum, trace, Frobenius and the spectral norm are denoted respectively by $\sigma_{\min}(A)$, $\sigma_{\max}(A)$, $\sigma(A)$, ${\rm trace}(A)$, $\|A\|_F$, and $\|A\|_2$. $I_n$ denotes the $n\times n$ identity matrix. Planted weights are denoted with an asterisk, e.g. $W^*$. $\exp(\alpha)$ denotes $e^\alpha$. Given any $v\in\mathbb{R}^n$, $\|v\|_2$ denotes its Euclidean $\ell_2$ norm $\sqrt{\sum_{1\le i\le n}v_i^2}$. Given two vectors $x,y\in\R^n$, their Euclidean inner product $\sum_{1\le i\le n}x_iy_i$ is denoted by $\ip{x}{y}$. Given a collection $Z_1,\dots,Z_k$ of objects of the same kind (e.g., vectors or matrices), ${\rm span}(Z_i:i\in[k])$ is the set, $\left\{\sum_{j=1}^k \alpha_j Z_j:\alpha_j\in\mathbb{R}\right\}$. $\Theta(\cdot),O(\cdot)$, $o(\cdot)$, and $\Omega(\cdot)$ are standard (asymptotic) order notations for comparing the growth of two sequences. $\risk{\cdot}$, $\nabla \risk{\cdot}$, $\mathcal{L}$, and $\nabla \mathcal{L}$ denote respectively the empirical risk,  its gradient; the population risk, and its gradient.
\section{Main Results}\label{sec:main-results}
\subsection{Optimization Landscape}\label{sec:landscape}
\subsubsection*{Existence of an Energy Barrier}
Our first result shows the appearance of an energy barrier in the landscape of the empirical risk $\risk{\cdot}$ below which any rank-deficient $W\in\R^{m\times d}$ ceases to exist, with high probability. 
\begin{theorem}\label{thm:energy-barrier-empirical}
Let $X_i\in\R^d$, $1\leqslant i\leqslant N$ be a collection of i.i.d. random vectors each having centered i.i.d. sub-Gaussian coordinates. That is, for some $C>0$, $\mathbb{P}(|X_i(j)|>t)\leqslant \exp(-Ct^2)$ for every $t\geqslant 0$, $i\in[N],j\in[d]$. Suppose, furthermore, that for every $M>0$, the distribution of $X_i(j)$, conditional on $|X_i(j)|\leqslant M$ is centered. Let $Y_i=f(W^*;X_i)$, $1\leqslant i\leqslant N$ be the corresponding label generated by a planted teacher network per (\ref{eq:fnc-nn-computes}), where ${\rm rank}(W^*)=d$ and $\|W^*\|_F\leqslant d^{K_2}$ for some $K_2>0$. Fix any $K_1>0$. Then, for some absolute constants $C,C_3,C'>0$, with probability at least
\[
1-\exp(-C'N)-\left(9d^{4K_1+4K_2+3}\right)^{d^2-1}\left(\exp\left(-C_3 Nd^{-4K_1-4K_2-2}\right)+Nde^{-Cd^{2K_1}}\right)
\]
it holds that  
\[
\min_{W\in\R^{m\times d}:{\rm rank}(W)\leqslant d-1} \risk{W}\triangleq 
\min_{W\in\R^{m\times d}:{\rm rank}(W)\leqslant d-1} \frac1N\sum_{1\leqslant i\leqslant N}(Y_i-f(W;X_i))^2 \geqslant \frac12C_5 \sigma_{\min}(W^*)^4.
\]
Here, $C_5 = \min\{\mu_4(K_1)-\mu_2(K_1)^2,2\mu_2(K_1)^2\}$, where $\mu_t(K)=\mathbb{E}[X_1(1)^t \mid |X_1(1)|\leqslant d^{K}]$.
\end{theorem}
Namely, with high probability, $\risk{W}$ is bounded away from zero by an explicit constant for any $W$ that is rank-deficient,  
provided $N=d^{O(1)}$. Several remarks are now in order. The assumption that the conditional mean of $X_i(j)$ is zero is benign: it holds, e.g., for zero-mean Gaussian variables. An inspection of the proof of Theorem \ref{thm:energy-barrier-empirical} reveals that the result still remains true even when the data coordinates has heavier tails, that is $\mathbb{P}(|X_i(j)|>t)\leqslant \exp(-Ct^\alpha)$ for any constant $\alpha$. 

The proof of Theorem \ref{thm:energy-barrier-empirical} is provided in Section \ref{sec:pf-thm:energy-barrier-empirical}.

Our next result is an analogue of Theorem \ref{thm:energy-barrier-empirical} for the population risk $\mathcal{L}(\cdot)$.
\begin{theorem}\label{thm:band-gap}
Suppose that $X\in\R^d$ has i.i.d. centered coordinates with variance $\mu_2$, (finite) fourth moment $\mu_4$, ${\rm rank}(W^*)=d$, and let $\mathcal{L}(W)=\mathbb{E}[(f(W;X)-f(W^*;X))^2]$.
\begin{itemize}
    \item[(a)] It holds that
    \[
       \min_{W\in\mathbb{R}^{m\times d}:{\rm rank}(W)<d} \mathcal{L}(W)\geqslant \min\{\mu_4-\mu_2^2,2\mu_2^2\}\cdot \sigma_{\min}(W^*)^4.
    \]
    \item[(b)] There  exists a matrix $W\in\R^{m\times d}$ such that ${\rm rank}(W)\leqslant d-1$ and 
    \[
    \mathcal{L}(W) \leqslant  \max\left\{\mu_4,3\mu_2^2\right\}\cdot\sigma_{\min}(W^*)^4.
    \]
\end{itemize}
\end{theorem}
The proof of Theorem \ref{thm:band-gap} is deferred to Section \ref{sec:proof-of-band-gap}. 

Two remarks are in order. First, the hypothesis of Theorem \ref{thm:band-gap} holds under a milder assumption on data. Second, part $({\rm b})$ of Theorem \ref{thm:band-gap} implies that our lower bound on the energy value is tight up to a multiplicative constant determined by the moments of the data.

As a simple corollary to Theorems \ref{thm:energy-barrier-empirical} and \ref{thm:band-gap}, we obtain that the landscape of the risks still admit an energy barrier, even if we consider the same network architecture with planted weight matrix  $W^*\in\R^{m\times d}$, and quadratic activation function having lower order terms, that is, the activation $\widetilde{\sigma}(x)=\alpha x^2+\beta x+\gamma$, with $\alpha\neq 0$. This barrier is quantified by $\alpha$, in addition to $\sigma_{\min}(W^*)$ and the corresponding moments of the data.
\begin{coro}\label{coro-1}
For any $W\in\R^{m\times d}$, define $\widetilde{f}(W;X) = \sum_{j=1}^m \widetilde{\sigma}(\ip{W_j}{X})$, where $\widetilde{\sigma}(x)=\alpha x^2+\beta x+\gamma$ with $\alpha,\beta,\gamma\in\R$ arbitrary. 
\begin{itemize}
    \item[(a)] The hypothesis of Theorem \ref{thm:energy-barrier-empirical} still holds with $f$ replaced with $\widetilde{f}$, and energy barrier $\frac12 C_5\sigma_{\min}(W^*)^4$ replaced with $\frac{\alpha^2}{2} C_5\sigma_{\min}(W^*)^4$.
    \item[(b)] The hypothesis of Theorem \ref{thm:band-gap}{\rm (a)} still holds with $f$ replaced with $\widetilde{f}$, and energy barrier $\min\{\mu_4-\mu_2^2,2\mu_2^2\}\cdot \sigma_{\min}(W^*)^4$ replaced with $\alpha^2 \min\{\mu_4-\mu_2^2,2\mu_2^2\}\cdot \sigma_{\min}(W^*)^4$.
\end{itemize}
\end{coro}
 The proof of this corollary is deferred to Section \ref{sec:proof-of-coro-1}.
\subsubsection*{Global Optimality of Full-Rank Stationary Points}
Our next result establishes that if $W$ is a full-rank stationary point of the empirical risk, and $N\geqslant d(d+1)/2$, then $W$ is necessarily a global minimum.
\begin{theorem}\label{thm:full-rank-empirical-global-opt}
Let $X_i\in\R^d$, $1\leqslant i\leqslant N$; $W^*\in\R^{m\times d}$ with ${\rm rank}(W^*)=d$, and suppose $W$ is a full-rank stationary point of the empirical risk: ${\rm rank}(W)=d$, and $\nabla_W \risk{W}=0$. Then, $\risk{W}=0$. Furthermore, if  $N\geqslant d(d+1)/2$, then $W=QW^*$ for some orthogonal matrix $Q\in\R^{m\times m}$.
\end{theorem}
The proof of Theorem \ref{thm:full-rank-empirical-global-opt} is given in Section \ref{sec:pf-thm:full-rank-empirical-global-opt}. Our next result is an analogue of Theorem \ref{thm:full-rank-empirical-global-opt} for the population risk, and requiring a milder distributional assumption.
\begin{theorem}\label{thm:full-rank-global-opt}
Suppose $W^*\in \R^{m\times d}$ with ${\rm rank}(W^*)=d$. Suppose $X\in\R^d$ has centered i.i.d. coordinates with $\E{X_i^2}=\mu_2$, $\E{X_i^4}=\mu_4$; and ${\rm Var}(X_i^2)>0$. Let $W\in\R^{m\times d}$ be a stationary point of the population risk with full-rank, that is, $\nabla \mathcal{L}(W) = \mathbb{E}[\nabla (f(W^*;X)-f(W;X))^2]=0$, and ${\rm rank}(W)=d$. Then, $W=QW^*$ for some orthogonal matrix $Q$, and that, $\mathcal{L}(W)=0$.
\end{theorem}
The proof of Theorem \ref{thm:full-rank-global-opt} is deferred to Section \ref{sec:proof-of-full-rank-global-opt}. 
Note that an implication of Theorems \ref{thm:full-rank-empirical-global-opt} and \ref{thm:full-rank-global-opt} is that the corresponding losses admit no rank-deficient saddle points. Namely, the landscape of the corresponding losses has fairly benign properties below the aforementioned energy barrier. We show how this implies the convergence of gradient descent in the next section.
\subsubsection*{Convergence of Gradient Descent}
We now combine Theorems \ref{thm:energy-barrier-empirical} and \ref{thm:full-rank-empirical-global-opt} to obtain the following potentially interesting conclusion on running the gradient descent for the empirical risk. Suppose, that the gradient descent algorithm is initialized at a point with sufficiently small  empirical risk, in particular lower than the smallest risk value achieved by rank-deficient 
matrices. Then, with a properly chosen step size; it finds an approximately stationary point $W$ (that is, $\|\nabla \risk{W}\|_F\leqslant \epsilon$) in time ${\rm poly}(\epsilon^{-1},d)$ for which the weights $W^TW$ are uniformly $\epsilon-$close to planted weights $(W^*)^TW^*$, and consequently the generalization error $\mathcal{L}(W)$ is at most (order) $\epsilon$. Furthermore, the algorithm converges to a global optimum of the empirical risk minimization problem $\min_W\risk{W}$, which is zero; thus recovering planted weights, due to the absence of spurious stationary points within the set of full-rank matrices.
\begin{theorem}\label{thm:gd-conv-empirical}
Suppose that $X_i\in\R^d$, $1\le i\le N$ enjoys the assumptions in Theorem~\ref{thm:energy-barrier-empirical}; $W_0\in\R^{m\times d}$ is a matrix of weights with the property
\[
\risk{W_0}<\frac12 C_5  \sigma_{\min}(W^*)^4,
\]
where $C_5$ is the constant defined in Theorem \ref{thm:energy-barrier-empirical}; and $\|W^*\|_F\leqslant d^{K_2}$. Define  
\[
L\triangleq \sup\left\{\|\nabla^2 \risk{W}\|:\risk{W}\leqslant \risk{W_0}\right\}
\]
where by $\|\nabla^2 \risk{W}\|$ we denote the spectral norm of the (Hessian) matrix $\nabla^2 \risk{W}$. Then, there exists an event of probability at least
    $$
    1-\exp(-c'N^{1/4})-(9d^{4K_1+4K_2+3})^{d^2-1}\left(\exp(-C_4 Nd^{-4K_1-4K_2-2}) + Nd\exp(-Cd^{2K_1})\right),
    $$
    (where $c',C,C_4>0$ are absolute constants) on which the following holds.
\begin{itemize}
    \item[(a)] 
    For any $W$ with $\risk{W}\leqslant \risk{W_0}$;
    $\|W\|_F\leqslant d^{K_2+1}$, and $L={\rm poly}(d)<+\infty$.
    \item[(b)] Running gradient descent with a step size of $0<\eta<1/2L$ generates a full-rank $\epsilon-$approximate stationary point $W\in\R^{m\times d}$ with $\|\nabla \risk{W}\|_F\leqslant \epsilon$
    in time ${\rm poly}(\epsilon^{-1},d)$. Furthermore, for this $W$, $\risk{W}\leqslant 32\epsilon \sigma_{\min}(W^*)^{-2}d^{4K_2+4}$.
    \item[(c)] For $W$ found in bullet ${\rm (b)}$,  it holds that $\|W^TW-(W^*)^TW^*\|_F\leqslant C'\sqrt{\epsilon}d^{K_1+2K_2+7}\sigma_{\min}(W^*)^{-1}$ (here $C'>0$ is some absolute constant); and consequently the generalization error $\mathcal{L}(W)$ is at most $2(C')^2\mu_2^2\epsilon d^{2K_1+4K_2+15}\sigma_{\min}(W^*)^{-1}$, provided $N\geqslant d^{18+\frac{8K_1}{3}}$.
    \item[(d)] Gradient descent algorithm with initialization $W_0\in\R^{m\times d}$ and a step size of $0<\eta<1/2L$ generates a trajectory $\{W_k\}_{k\geqslant 0}$ of weights such that 
    $
    \lim_{k\to\infty}\risk{W_k}=\min_W \risk{W}=0$.
\end{itemize}

\end{theorem}
We note that the exponent $1/4$ in the probability and the sample bound $d^{18+\frac{8K_1}{3}}$ are required only for part {\rm (c)}, and can potentially be improved. In particular, the exponent can be improved to one for parts {\rm (a),(c)} and {\rm (d)}. 

We now provide an important remark pertaining {\rm (c)}: provided $N$ grows at least polynomially in $d$, with probability $1-\exp(-C'N^{1/4})$ it holds that for any $W$ with $\risk{W}\leqslant \kappa$, $W^TW$ is close to $(W^*)^TW^*$, that is $\|W^TW-(W^*)^TW^*\|_F\leqslant d^{O(1)}\sqrt{\kappa}$; and consequently $\mathcal{L}(W)\leqslant d^{O(1)}\kappa$. To the best of our knowledge, this is a novel contribution of ours, and is achieved by controlling condition number of a certain  matrix with i.i.d. rows consisting of tensorized data $X_i^{\otimes 2}$; using a very recent  work analyzing the spectrum of expected covariance matrices of tensorized data \cite{emschwiller2020neural}.

The proof of Theorem \ref{thm:gd-conv-empirical} is provided in Section \ref{sec:pf-thm:gd-conv-empirical}. 

By combining Theorems \ref{thm:band-gap} and \ref{thm:full-rank-global-opt}, we obtain an
analogous result for the population risk:
\begin{theorem}\label{thm:gd-conv}
Let $W_0\in\R^{m\times d}$ be a matrix of weights, with the property that
\[
\mathcal{L}(W_0)<\min_{W\in\R^{m\times d}:{\rm rank}(W)<d} \mathcal{L}(W).
\]
Define
\[
L = \sup\left\{\|\nabla^2 \mathcal{L}(W)\|:\mathcal{L}(W)\leqslant \mathcal{L}(W_0)\right\},
\]
where by $\|\nabla^2 \mathcal{L}(W)\|$ we denote the spectral norm of the matrix $\nabla^2 \mathcal{L}(W)$. Then, $L<+\infty$ and the gradient descent algorithm with initialization $W_0\in\R^{m\times d}$ and a step size of $0<\eta<1/2L$ generates a trajectory $\{W_k\}_{k\geqslant 0}$ of weights such that $\lim_{k \rightarrow \infty} \mathcal{L}(W_k)= \min_W \mathcal{L}(W)=0$. 
\end{theorem}
The proof of Theorem \ref{thm:gd-conv} is provided in Section \ref{sec:pf-thm:gd-conv}.

The above result concerns the performance of gradient descent  assuming the initialization is proper, i.e. it is below the aforementioned energy barrier. One can then naturally ask whether such an initialization is indeed possible in some generic context. In the next section, we address this question of proper initialization when the (planted) weights are generated randomly, to complement Theorems~\ref{thm:gd-conv-empirical} and \ref{thm:gd-conv}. We establish that such a proper initialization is indeed possible by providing a deterministic initialization guarantee, which with high probability beats the aforementioned energy barrier.
\subsection{On Initialization: Randomly Generated Planted Weights}\label{sec:random-initialization}
As noted in the previous section, our results offer an alternative conceptual explanation for the success of training gradient descent in learning aforementioned neural network architectures from the landscape perspective; provided that the algorithm is initialized properly. 

In this section, we provide a way to properly initialize such networks under the assumption that the data has centered i.i.d. sub-Gaussian coordinates; and the (planted) weight matrix $W^*\in\R^{m\times d}$ has i.i.d. centered entries with unit variance and finite fourth moment. Our result is valid provided that the network is sufficiently overparametrized: $m>Cd^2$ for some large constant $C$. Note that this implies $W^*$ is a tall matrix sending $\R^d$ into $\R^m$. 
The rationale behind this approach is as follows: the value of the risk is determined by the spectrum of $\Delta\triangleq W^TW-(W^*)^TW^*$ and the moments of the data distribution. Furthermore, under the randomness assumption, the Wishart matrix $(W^*)^TW^*$ is tightly concentrated around a multiple of the identity if $m$ is sufficiently large. Hence one can control the spectrum of $\Delta$, and therefore the risk $\risk{\cdot}$, by properly choosing $W$.

Equipped with these observations, we are now in a position to state our result, a high probability guarantee for the cost of a particular choice of initialization.

\begin{theorem}\label{thm:initialization-empirical}
Suppose that the planted weight matrix $W^*\in\R^{m\times d}$ has centered i.i.d. entries with unit variance and finite fourth moment; the (i.i.d.) data $X_i\in\R^d$, $1\leqslant i\leqslant  N$, has i.i.d. centered sub-Gaussian coordinates; and the $W_0\in\R^{m\times d}$ satisfies $(W_0)_{ii}=\sqrt{m}$ for $i\in[d]$ and $(W_0)_{ij}=0$ for $i\neq j$ (namely $W_0^T W_0 = mI_d\in\R^{d\times d}$). Then for some absolute constants $C,C'>0$ with probability at least
\[
1-\exp\left(-C'\frac{N}{d^{4K+3}m}\right)-Nd\exp(-Cd^{2K}) -o_d(1),
\]
it is the case that for the constant $C_5$ defined in Theorem \ref{thm:energy-barrier-empirical},
\[
\risk{W_0}<\frac12 C_5 \sigma_{\min}(W^*)^4
\]
provided $m>Cd^2$ for a sufficiently large constant $C>0$.
\end{theorem}
The proof of Theorem \ref{thm:initialization-empirical} is provided in Section \ref{sec:pf-thm:initialization-empirical}.

The corresponding result for the population risk is provided below.


\begin{theorem}\label{thm:initialization}
Suppose that the data $X\in\R^d$ consists of i.i.d. centered coordinates with ${\rm Var}(X_i^2)>0$ and $\E{X_i^4}<\infty$. Recall that
$$
\mathcal{L}(W) =  \mathbb{E}\left[\left(f(W;X)-f(W^*;X)\right)^2\right],
$$
where the expectation is taken with respect to the randomness in a fresh sample $X$.
\begin{itemize}
\item[(a)] Suppose that the planted weight matrix $W^*\in\R^{m\times d}$ has i.i.d. standard normal entries. Let the initial weight matrix $W_0\in\R^{m\times d}$ be defined by $(W_0)_{i,i}=\sqrt{m+4d}$ for $1\leqslant i\leqslant d$, and $(W_0)_{i,j}=0$ otherwise (hence, $W_0^T W_0 = \gamma I_d$ with $\gamma=m+4d$). Then, provided $m>Cd^2$ for a  sufficiently large absolute constant $C>0$, 
$$
\mathcal{L}(W_0)< \min_{W\in\R^{m\times d}:{\rm rank}(W)<d} \mathcal{L}(W),
$$
with probability at least $1-\exp(-\Omega(d))$, where the probability is with respect to the draw of $W^*$.
\item[(b)] Suppose the planted weight matrix $W^*\in\R^{m\times d}$ has centered i.i.d. entries with  unit variance and finite fourth moment.  Let the initial weight matrix $W_0\in\R^{m\times d}$ be defined by $(W_0)_{i,i}=\sqrt{m}$ for $1\leqslant i\leqslant d$, and $(W_0)_{i,j}=0$ otherwise (hence, $W_0^T W_0 = mI_d$). Then, provided $m>Cd^2$ for a  sufficiently large absolute constant $C>0$, 
$$
\mathcal{L}(W_0)< \min_{W\in\R^{m\times d}:{\rm rank}(W)<d} \mathcal{L}(W),
$$
with high probability, as $d\to\infty$, where the probability is with respect to the draw of $W^*$.
\end{itemize}
\end{theorem}
The proof of this theorem is provided in Section \ref{sec:pf-initialization}. 

Note that, the part ${\rm (a)}$ of Theorem \ref{thm:initialization} gives an explicit rate for probability, in the case when the i.i.d. entries of the planted weight matrix $W^*$ are standard normal, and is based on a non-asymptotic concentration  result for the spectrum of such matrices. The extension in part ${\rm (b)}$ is based on a result of Bai and Yin \cite{bai1988convergence}. 



With this, we now turn our attention to the number of training samples required to learn such models.
\subsection{Critical Number of Training Samples}\label{sec:data-no}
The focus of previous sections is on landscape results pertaining the empirical risk minimization problem. One can then naturally ask the following question: what is the smallest number of samples required to claim that small empirical risk also controls the generalization error?

In this section, our focus is on the number of training samples required for controlling the generalization error. We identify a necessary and sufficient condition on the training data under which any minimizer of the empirical risk (which, in the case we consider of planted weights, necessarily interpolates the data) has zero generalization error. We obtain our results for potentially overparametrized interpolators, that is of potentially larger width than the width of the original network generating the weights. Furthermore we identify the smallest number $N^*$ of training samples, such that (randomly generated) training data $X_1,\dots,X_N$ satisfies the aforementioned condition, so long as $N\geqslant N^*$.


\subsubsection*{A Necessary and Sufficient Geometric Condition on the Training Data}
We start by providing a necessary and sufficient (geometric) condition on the training data under which any minimizer of the empirical risk (which, in the case of planted weights, necessarily interpolates the data) has zero generalization error.
\begin{theorem}\label{thm:geo-condition}
Let $X_1,\dots,X_n\in \R^d$ be a set of data.
\begin{itemize}
    \item[(a)] Suppose
    \[
    {\rm span}\{X_iX_i^T : 1\leqslant i\leqslant N\}= \mathcal{S},
    \]
    where $\mathcal{S}$ is the set of all $d\times d$ symmetric real-valued matrices. Let $\widehat{m}\in\mathbb{N}$ be arbitrary. Then for any $W\in\mathbb{R}^{\widehat{m}\times d}$ interpolating the data, that is $f(W^*;X_i)=f(W;X_i)$ for every $i\in[N]$, it holds that $W^T W=(W^*)^T W^*$. In particular, if $\widehat{m}\geqslant m$, then for some matrix $Q\in\mathbb{R}^{\widehat{m} \times m}$ with orthonormal columns, $W=QW^*$, and if $m\geqslant \widehat{m}$, then for some matrix $Q'\in\R^{m\times \widehat{m}}$ with orthonormal columns, $W^*=Q'W$. 
    \item[(b)] Suppose, 
    \[
    {\rm span}\{X_iX_i^T : 1\leqslant i\leqslant N\},
    \]
    is a strict subset of $\mathcal{S}$. Then, for any $W^*\in\mathbb{R}^{m\times d}$ with ${\rm rank}(W^*)=d$ and any positive integer $\widehat{m}\geqslant d$, there exists a $W\in\mathbb{R}^{\widehat{m}\times d}$ such that $W^T W \neq (W^*)^T W^*$, while $W$ interpolates the data,  that is, $f(W^*;X_i)=f(W;X_i)$ for all $i\in[N]$. In particular, for this $W\in\mathbb{R}^{m\times d}$, $\mathcal{L}(W)>0$,
    where $\mathcal{L}$ is defined with respect to any jointly continuous distribution on $\R^d$. 
\end{itemize}
\end{theorem}
The proof of Theorem \ref{thm:geo-condition} is deferred to Section \ref{sec:proof-of-geo-cond}. 

Several remarks are now in order. The condition stated in Theorem \ref{thm:geo-condition} is not retrospective in manner: it can be checked ahead of the optimization process. Next, there are no randomness assumptions in the setting of Theorem \ref{thm:geo-condition}, and it provides a purely geometric necessary and sufficient condition: as long as ${\rm span}(X_iX_i^T : i\in [N])$ is the space of all symmetric matrices (in $\R^{d\times d}$) we have that any (global) minimizer of the empirical risk has zero generalization error. Conversely, in the absence of this geometric condition, there are optimizers $W\in\R^{m\times d}$ of the empirical risk $\risk{\cdot}$ such that while $\risk{W}=0$, the generalization error of $W$ is bounded away from zero, that is, $W^T W\neq (W^*)^T W^*$. It is also worth recalling that in the case when $W$ does not interpolate the data but has a rather small training error, the result of Theorem~\ref{thm:gd-conv-empirical}{\rm (c)} allows one to control $\|W^TW-(W^*)^TW^*\|_F$, and consequently the generalization error $\mathcal{L}(W)$. Soon in Theorem \ref{thm:main}, we give a  more refined version of this result, with a concrete lower bound on $\mathcal{L}(W)$, in the more realistic setting, where the training data is generated randomly. 

We further highlight the presence of the parameter $\widehat{m}\in\mathbb{N}$. In particular, part $({\rm a})$ of Theorem \ref{thm:geo-condition} states that provided the span condition is satisfied, any neural network with $\widehat{m}$ internal nodes interpolating the data has necessarily zero generalization error, regardless of whether $\widehat{m}$ is equal to $m$,  in particular, even when $\widehat{m}\geqslant m$. This, in fact, is an instance of an interesting phenomenon empirically observed about neural networks, which somewhat challenges one of the main paradigms in statistical learning theory:  overparametrizartion does not hurt generalization performance of neural networks once the data is interpolated. Namely beyond the interpolation threshold, one retains good generalization property. 

We note that Theorem \ref{thm:geo-condition} still remains valid under a slightly more general setup, where each node $j\in[m]$ has an associated positive but otherwise arbitrary output weight $a_j^*\in\R_+$.
\begin{coro}\label{coro-2}
Let $W\in\R^{m\times d}$, $a\in\R_+^m$, and $\widehat{f}(a,W,X)$ be the function computed by the neural network with input $X\in\R^d$, quadratic activation function, planted weights $W\in\R^{m\times d}$, and output weights $a\in\R_+^m$, that is, 
$\widehat{f}(a,W,X) = \sum_{j=1}^m a_j \ip{W_j}{X}^2$. Let $X_1,\dots,X_n\in\R^d$ be a set of data.
\begin{itemize}
	\item[(a)] Suppose, 
	\[
	{\rm span}\{X_iX_i^T:1\leqslant i\leqslant N\} = \mathcal{S}.
	\]
	Then for any $\widehat{m}\in\mathbb{N}$ and $(a,W)\in\R_+^{\widehat{m}} \times \R^{\widehat{m}\times d}$ interpolating the data, that is $\widehat{f}(a^*,W^*,X_i)= \widehat{f}(a,W,X_i)$ for every $i\in[N]$, it holds that $\widehat{f}(a,W,X)=\widehat{f}(a^*,W^*,X)$ for every $X\in\R^d$ (here, $a_j^*>0$ for all $j$). In particular, $(a,W)$ achieves zero generalization error. 
	\item[(b)] Suppose
	\[
	{\rm span}\{X_iX_i^T:1\leqslant i\leqslant N\} 
	\] 
	is a strict subset of $\mathcal{S}$. Then, for any $(a^*,W^*)\in\R_+^m\times \R^{m\times d}$, and every $\widehat{m}\geqslant d$, there is a pair $(a,W)\in\R_+^{\widehat{m}}\times \R^{\widehat{m} \times d}$, such that while $(a,W)$ interpolates the data, that is, $\widehat{f}(a,W,X_i)=\widehat{f}(a^*,W^*,X_i)$ for every $i\in[N]$, $(a,W)$ has strictly positive generalization error, with respect to any jointly continuous distribution on $\R^d$.
\end{itemize}
\end{coro}
The proof of this corollary is deferred to Section \ref{sec:proof-of-coro-2}.
\subsubsection*{Randomized Data Enjoys the Geometric Condition}
We now identify the smallest number $N^*$ of training samples, such that (randomly generated) training data $X_1,\dots,X_N$ satisfies the aforementioned geometric condition almost surely; as soon as $N\geqslant N^*$.
\begin{theorem}\label{thm:random-data-geo-cond}
Let $N^* = d(d+1)/2$, and $X_1,\dots,X_N \in \R^d$ be i.i.d. random vectors with jointly continuous distribution. Then,
\begin{itemize}
    \item[(a)] If $N\geqslant N^*$, then $\mathbb{P}({\rm span}(X_iX_i^T : i\in [N])=\mathcal{S})=1$.
    \item[(b)] If $N<N^*$, then for arbitrary $Z_1,\dots,Z_N\in\R^d$, ${\rm span}(Z_iZ_i^T:i\in[N])\subsetneq \mathcal{S}$. 
    \end{itemize}
\end{theorem}
The proof of Theorem \ref{thm:random-data-geo-cond} is deferred to Section \ref{sec:proof-of-random-data-geo}.

The critical number $N^*$ is obtained to be $d(d+1)/2$ since ${\rm dim}(\mathcal{S})= \binom{d}{2}+d=d(d+1)/2$. Note also that, with this observation, part $(b)$ of Theorem \ref{thm:random-data-geo-cond} is trivial, since we do not have enough number of matrices to span the space $\mathcal{S}$. 

\subsubsection*{Sample Complexity Bound for the Planted Network Model}
Combining Theorems \ref{thm:geo-condition} and \ref{thm:random-data-geo-cond}, we arrive at the following sample complexity result. 
\begin{theorem}\label{thm:main}
Let $X_i, 1\leqslant i\leqslant N$ be i.i.d. with a jointly continuous distribution on $\R^d$. 
Let the corresponding outputs $(Y_i)_{i=1}^N$ be generated via $Y_i = f(W^*;X_i)$, with $W^*\in\R^{m\times d}$  with  ${\rm rank}(W^*)=d$. 
\begin{itemize}
    \item[(a)] Suppose $N\geqslant  N^*$, and $\widehat{m}\in\mathbb{N}$. Then with probability one over the training data $X_1,\dots,X_n$, if $W\in\mathbb{R}^{\widehat{m}\times d}$ is such that $f(W;X_i)=Y_i$ for every $i\in[N]$, then $f(W;X)=f(W^*;X)$ for every $X\in\R^d$.
    \item[(b)] Suppose $X_i, 1\leqslant i\leqslant N$ are i.i.d. random vectors with i.i.d. centered coordinates having variance $\mu_2$ and finite fourth moment $\mu_4$. Suppose that $N<N^*$. Then there exists a $W\in\R^{m\times d}$ such that $f(W;X_i)=Y_i$ for every $i\in[N]$, yet the generalization error satisfies
    \begin{align*}
    \mathcal{L}(W) \geqslant \min\{\mu_4-\mu_2^2,2\mu_2^2\}\sigma_{\min}(W^*)^4.
    \end{align*} 
\end{itemize}
\end{theorem}
The proof of Theorem \ref{thm:main} is deferred to Section \ref{sec:proof-of-thm-mainn}. 

We highlight that the lower bound arising in Theorem \ref{thm:main} ${\rm (b)}$ is very similar to the energy barrier bounds obtained earlier for rank-deficient matrices in Theorem \ref{thm:band-gap} ${\rm (a)}$ and Theorem~\ref{thm:energy-barrier-empirical}. 
Note also that the interpolating network in in part ${\rm (a)}$ can potentially be larger than the original network generating the data: any large network, despite being overparametrized, still generalizes well, provided it interpolates on a training set enjoying the aforementioned geometric condition. 

Theorems \ref{thm:geo-condition} and \ref{thm:main} together provide the necessary and sufficient number of data points for training a shallow neural  network with quadratic activation function so as to guarantee good (perfect) generalization property. 

\section{Auxiliary Results}\label{sec:auxi}
We collect herein several useful auxiliary results that we  utilize in our proofs. The proofs of these auxiliary results  are provided in Section \ref{sec:pf-thm:analytic-exp-pop-risk}.

\subsection{An Analytical Expression for the Population Risk}
Towards proving our energy barrier results, Theorem \ref{thm:energy-barrier-empirical} and Theorem \ref{thm:band-gap}, we start with
providing an analytical expression for the population risk $\mathcal{L}(W)$ of any $W\in\mathbb{R}^{m\times d}$ in terms of how close it is to the planted weight matrix $W^*\in\R^{m\times d}$. 

We recall that a random vector $X$ in $\R^d$ is defined to have jointly continuous distribution
if there exists a measurable function $f:\R^d\to \R$ such that for any $i\in [N]$ and Borel set $\mathcal{B}\subseteq \R^d$, 
\begin{align*}
\mathbb{P}(X \in \mathcal{B})=\int_{\mathcal{B}}f(x_1,\dots,x_d)\;d\lambda(x_1,\dots,x_d),
\end{align*}
where $\lambda$ is the Lebesgue measure on $\R^d$. 
\begin{theorem}\label{thm:analytic-exp-pop-risk}
Let $W^*\in \R^{m\times d}$, $f(W^*;X)$ be the function computed by (\ref{eq:fnc-nn-computes}); and $f(W;X)$ be similarly the function computed by (\ref{eq:fnc-nn-computes}) for $W\in\R^{m\times d}$. Recall,
$$
\mathcal{L}(W)=\mathbb{E}[(f(W^*;X)-f(W;X))^2],
$$
where the expectation is with respect to the distribution of $X\in\R^d$.
\begin{itemize}
    \item[(a)] 
    Suppose the distribution of $X$ is jointly continuous. Then
    $\mathcal{L}(W)=0$, that is, $f(W^*;X)=f(W;X)$ almost surely with respect to $X$, if and only if $W=QW^*$ for some orthonormal matrix $Q\in\mathbb{R}^{m\times m}$.
\end{itemize}
Suppose now that the coordinates of $X\in\R^d$ are i.i.d. with $\E{X_i}=0,\E{X_i^2}=\mu_2$, and $\E{X_i^4}=\mu_4$.
\begin{itemize}
\item[(b)] It holds that:
$$
\mathcal{L}(W)  = \mu_2^2 \cdot {\rm trace}(A)^2 +2\mu_2^2 \cdot {\rm trace}(A^2) +(\mu_4-3\mu_2^2)\cdot {\rm trace}(A\circ A),
$$
where $A=(W^*)^T W^* - W^T W\in\R^{d\times d}$, and $A\circ A$ is the Hadamard product of $A$ with itself. In particular, if $X\in\R^d$ has i.i.d. standard normal coordinates, we obtain $\mathcal{L}(W) = {\rm trace}(A)^2+2{\rm trace}(A^2)$.
\item[(c)] The following bounds hold:
$$
\mu_2^2 \cdot {\rm trace}(A)^2 +  \min\left\{\mu_4 -\mu_2^2 ,2\mu_2^2 \right\} \cdot{\rm trace}(A^2)\leqslant \mathcal{L}(W),
$$
and
$$\mu_2^2 \cdot {\rm trace}(A)^2 +  \max\left\{\mu_4 -\mu_2^2 ,2\mu_2^2 \right\}\cdot {\rm trace}(A^2)\geqslant \mathcal{L}(W).$$
\end{itemize}
\end{theorem}

In a nutshell, Theorem \ref{thm:analytic-exp-pop-risk} states that the population risk $\mathcal{L}(W)$ of any $W\in\mathbb{R}^d$ is completely determined by how close it is to the planted weights $W^*$ as measured by the matrix $A=(W^*)^T W^* - W^T W$; and the second and fourth moments of the data. This is not surprising: $\mathcal{L}(W)$ is essentially a function of the first four moments of the data, and the difference of the quadratic forms generated by $W$  and $W^*$, which is precisely encapsulated by the matrix $A$. Note also that the characterization of the ``optimal orbit" per part $({\rm a})$ is not surprising either:  any matrix $W$ with the property $W=QW^*$ where $Q\in\R^{m\times m}$ is an orthonormal matrix, that is, $Q^T Q =I_m$, has the property that $f(W;X)=\|WX\|_2^2=X^TW^TWX=f(W^*;X)$ for any data $X\in\R^d$. 
Part $({\rm a})$ then says the the reverse is true as well, provided that the distribution of $X$ is jointly continuous. Note also that for $X$ with centered i.i.d.\ entries the thesis of part ${\rm  (a)}$ follows also from part ${\rm (c)}$: $\mathcal{L}(W)=0$ implies that ${\rm trace}(A^2)=0$, which, together with the fact that $A$ is symmetric, then yields $A=0$, that is, $W^T W=(W^*)^T W^*$. 
\subsection{Useful Lemmas and Results from Linear Algebra and Random Matrix Theory}
Our next result is a simple norm bound for the ensemble $X_i\in\R^d$, $1\leqslant i\leqslant N$ with sub-Gaussian coordinates.
\begin{lemma}\label{lemma:bd-data}
Let $X_i\in\R^d$, $1\leqslant i\leqslant N$ be an i.i.d. collection of random vectors with centered i.i.d. sub-Gaussian coordinates, that is, for some constant $C>0$,  $\mathbb{P}(|X_i(j)|>t)\leqslant\exp(-Ct^2)$ for every $i\in[N],j\in[d]$, and $t\geqslant 0$. Then,
\[
\mathbb{P}\left(\left\|X_i\right\|_\infty<d^{K_1},1\leqslant i\leqslant N\right)\geqslant 1-Nd\exp(-Cd^{2K_1}).
\]
\end{lemma}
Our energy barrier result Theorem \ref{thm:energy-barrier-empirical} for the empirical risk is proven by establishing the emergence of a barrier for a {\bf single} rank-deficient $A\in\R^{d\times d}$, together with a covering numbers argument. 
\begin{lemma}\label{lemma:single-concentration}
Let $X_i\in\R^d$, $1\leqslant i\leqslant N$ be a collection of i.i.d. data with centered i.i.d. sub-Gaussian coordinates where for any $M>0$, the mean of $|X_1(1)|$ conditional on $|X_1(1)|\leqslant M$ is zero; and let $Y_i=f(W^*;X_i)$ be the corresponding label generated by a neural network with planted weights $W^*\in\R^{m\times d}$ as per (\ref{eq:fnc-nn-computes}), where $\|W^*\|_F\leqslant d^{K_2}$. 
Fix any $A\in\R^{d\times d}$, where $\|A\|_F\leqslant d^{2K_2}$, ${\rm rank}(A)\leqslant d-1$, and $A\succeq 0$. Fix $K_1>0$ and define the event
\[
\mathcal{E}(A) \triangleq \left\{ \frac1N\sum_{1\leqslant i\leqslant N}\left(Y_i - X_i^T A X_i\right)^2 \geqslant \frac12 C_5 \sigma_{\rm min}(W^*)^4
\right\},
\]
where 
\[
C_5 = \min\{\mu_4(K_1)-\mu_2(K_1)^2,2\mu_2(K_1)^2\}
\]
where $\mu_n(K)=\mathbb{E}[X_1(1)^n\mid |X_1(1)|\leqslant d^{K}]$.
Then, there exists an constant $C'>0$ (independent of $W$,  and depending  only on data distribution, $K_1$, and $W^*$) such that 
\[
\mathbb{P}(\mathcal{E}(A))\geqslant 1-\exp\left(-C_3\frac{N}{d^{4K_1+4K_2+2}}\right)-Nde^{-Cd^{2K_1}},
\]
where $C>0$ is the same constant as in Lemma \ref{lemma:bd-data}.
\end{lemma}
The next result is a covering number bound, adopted from \cite[Lemma~3.1]{candes2011tight} with minor modifications.
\begin{lemma}\label{lemma:cov-number}
Let 
\[
S_R\triangleq \left\{A\in\R^{d\times d}:{\rm rank}(A)\leqslant r,A\succeq 0, \|A\|_F\leqslant R\right\}.
\]
Then there exists an $\epsilon-$net $\bar{S_R}$ for $S_R$ in Frobenius norm (that is, for every $A\in S_R$ there exists a $\widehat{A}\in\bar{S_R}$ such that $\|A-\widehat{A}\|_F\leqslant \epsilon$) such that
\[
|\bar{S}_R|\leqslant \left(\frac{9R}{\epsilon}\right)^{dr+r}.
\]
\end{lemma}

Some of our results use the following well-known results:
\begin{theorem}{(\cite{caron2005zero})}\label{thm:auxiliary}
Let $\ell$ be an arbitrary positive integer; and $P:\R^\ell\to \R$ be a polynomial. Then, either $P$ is identically $0$, or $\{x\in \R^\ell:P(x)=0\}$ has zero Lebesgue measure, namely, $P(x)$ is non-zero almost everywhere.
\end{theorem}

\begin{theorem} {(\cite[Theorem~7.3.11]{horn2012matrix})}\label{thm:auxiliary-2}
For two matrices $A \in\R^{p \times n}$ and $B\in\R^{q\times n}$ where $q\leqslant p$; $A^T A = B^T B$ holds if and only if $A=QB$ for some matrix $Q\in\R^{p\times q}$ with orthonormal columns. 
\end{theorem}
Our results regarding the initialization guarantees use the several auxiliary results from random matrix theory: The spectrum of tall random matrices are essentially concentrated:
\begin{theorem}\rm{(\cite[Corollary~5.35]{vershynin2010introduction})}\label{thm:tall-matrix-spectra-concentrate}
Let $A$ be an $m\times d$ matrix with independent standard normal entries. For every $t\geqslant 0$,  with probability at least $1-2\exp(-t^2/2)$, we have:
$$
\sqrt{m}-\sqrt{d}-t \leqslant \sigma_{\min}(A)\leqslant \sigma_{\max}(A)\leqslant \sqrt{m}+\sqrt{d}+t.
$$
\end{theorem}
\begin{theorem}{(\cite{bai1993limit},\cite[Theorem~5.31]{vershynin2010introduction})}\label{thm:baiyinvershy}

 Let $A = A_{N,n}$ be an $N\times n$ random matrix whose entries are independent copies of a random variable with zero mean, unit variance, and finite fourth moment. Suppose that the dimensions $N$ and $n$ grow to infinity while the aspect ratio $n/N$ converges to a constant in $[0, 1]$. Then
$$
\sigma_{\min}(A) =\sqrt{N}-\sqrt{n}+o(\sqrt{n}),\quad\text{and}\quad 
\sigma_{\min}(A) =\sqrt{N}+\sqrt{n}+o(\sqrt{n}),
$$
almost surely.
\end{theorem}
The following concentration result, recorded herein verbatim from Vershynin~\cite{vershynin2010introduction}, will be beneficial for our approximate stationarity analysis. 

\begin{theorem}(\cite[Theorem~5.44]{vershynin2010introduction})\label{thm:vershy}
Let $A$ be an $N\times n$ matrix whose rows $A_i$ are independent random vectors in $\R^n$ with the common second moment matrix $\Sigma=\mathbb{E}[A_iA_i^T]$. Let $m$ be a number such that $\|A_i\|_2\leqslant \sqrt{m}$ almost surely for all $i$. Then, for every $t\geqslant 0$, the following inequality holds with probability at least $1-n\cdot \exp(-ct^2)$:
$$
\left\|\frac{1}{N} A^T A - \Sigma\right\|\leqslant \max\left(\|\Sigma\|^{1/2}\delta,\delta^2\right)\quad\text{where}\quad \delta=t\sqrt{m/N}.
$$
Here, $c>0$ is an absolute constant. 
\end{theorem}

Finally, we make use of the matrix-operator version of the H\"{o}lder's inequality:
\begin{theorem}\label{thm:matrix-holder}
For any matrix $U\in\R^{k\times \ell}$, let $\|U\|_{\sigma_p}$ be the $\ell_p$ norm of the vector \[
(\sigma_1(U),\dots,\sigma_{\min\{k,\ell\}}(U))
\] of singular values of $U$. Then, for any $p,q>0$ with  $\frac1p+\frac1q=1$, it holds that
\[
|\ip{U}{V}|=|{\rm trace}(U^T V)|\leqslant \|U\|_{\sigma_p}\|V\|_{\sigma_q}.
\]
\end{theorem}

\section{Proofs}\label{sec:proofs}
In this section, we present the proofs of the main results of this paper.
\subsection{Proofs of Auxiliary Results}
\subsubsection*{Proof of Theorem  \ref{thm:analytic-exp-pop-risk}} \label{sec:pf-thm:analytic-exp-pop-risk} 
First, we have
\begin{equation}\label{eq:super-useful}
f(W;X)-f(W^*;X)= X^T((W^*)^T W^* - W^T W)X \triangleq X^T A X,
\end{equation}
where $A=(W^*)^T W^* - W^T W\in \R^{d\times d}$ is a symmetric matrix. Note also that, 
\begin{equation}\label{eq:tr-of-A-squared}
 {\rm trace}(A)^2 = \sum_{i=1}^d A_{ii}^2+2\sum_{i<j}A_{ii}A_{jj},
\end{equation} 
and
\begin{equation}\label{eq:tr-of-A2}
{\rm trace}(A^2)={\rm trace}(A^T A)=\|A\|_F^2 = \sum_{i,j}A_{ij}^2 = \sum_{i=1}^d A_{ii}^2 + 2\sum_{i<j}A_{ij}^2,
\end{equation}
where $A^2$ is equal to $A^T A$, as $A$ is symmetric.

\begin{itemize}
\item[(a)] Recall Theorem \ref{thm:auxiliary}. In particular, if $\mathcal{L}(W)=0$, then we have $P(X)=X^T A X=0$ almost surely. Since $P(\cdot):\R^d \to \R$ a polynomial, it then follows that $P(X)=0$ identically. Now, since $A$ is symmetric, it has real eigenvalues, called $\lambda_1,\dots,\lambda_d$ with corresponding (real) eigenvectors $\xi_1,\dots,\xi_d$. Now, taking $X=\xi_i$, we have $X^T A X = \xi_i^T A\xi = \lambda_i \ip{\xi_i}{\xi_i}=0$. Since $\xi_i\neq 0$, we get $\lambda_i=0$ for any $i$. Finally, since $A=Q\Lambda Q^T$, it must necessarily be the case that $A=0$. Hence, $W^T W = (W^*)^T W^*$, which imply $W=QW^*$ for some $Q\in\R^{m\times m}$ orthonormal, per Theorem \ref{thm:auxiliary-2}.
\item[(b)] Using Equation (\ref{eq:super-useful}), we first have
$$
\mathcal{L}(W) = \sum_{1\leqslant i,j,i',j'\leqslant d}A_{ij}A_{i',j'}\E{X_iX_jX_{i'}X_{j'}}.
$$
Note that if $|\{i,j,i',j'\}|  \in\{3,4\}$, then $\E{X_iX_jX_{i'}X_{j'}}=0$, since $X$ has centered i.i.d. coordinates. Keeping this in mind, and carrying out the algebra we then get:
\begin{align*}
\mathcal{L}(W) &= \sum_{i=1}^d A_{ii}^2 \E{X_i^4} + 2\sum_{i<j}A_{ii}A_{jj}\E{X_i^2}\E{X_j^2} +4\sum_{i<j}A_{ij}^2 \E{X_i^2}\E{X_j^2}\\
& = \mu_4\sum_{i=1}^d  A_{ii}^2 + 2\mu_2^2 \sum_{i<j}A_{ii}A_{jj}+4\mu_2^2 \sum_{i<j}A_{ij}^2.
\end{align*}
Using now Equations (\ref{eq:tr-of-A-squared}) and (\ref{eq:tr-of-A2}), we get:
$$
\mathcal{L}(W) = (\mu_4-3\mu_2^2)\cdot {\rm trace}(A\circ A) +\mu_2^2 \cdot {\rm trace}(A)^2+ 2\mu_2^2 \cdot {\rm trace}(A^2),
$$
since $A_{ii}^2 = (A\circ A)_{ii}$. 
\item[(c)] Define $k$ to be such that $\mu_4-\mu_2^2=2k\mu_2^2$, namely, $k$ is  related to measures of dispersion pertaining $X_i$: $\sqrt{2k}$ is the coefficient of variation and $(2k+1)$ is the kurtosis associated to the random variable $X_i$. With this, we have:
$$
\mathcal{L}(W) = \mu_2^2 \cdot {\rm trace}(A)^2 +2\mu_2^2 \left(k \sum_{i=1}^d A_{ii}^2 + 2\sum_{i<j}A_{ij}^2\right).
$$
From here, the desired conclusion follows since 
$$
\mu_2^2 \cdot {\rm trace}(A)^2 +2\min\{k,1\}\mu_2^2 \left( \sum_{i=1}^d A_{ii}^2 + 2\sum_{i<j}A_{ij}^2\right)\leqslant \mathcal{L}(W),
$$
and
$$
\mu_2^2 \cdot {\rm trace}(A)^2 +2\max\{k,1\}\mu_2^2 \left( \sum_{i=1}^d A_{ii}^2 + 2\sum_{i<j}A_{ij}^2\right)\geqslant  \mathcal{L}(W),
$$
together with Equation (\ref{eq:tr-of-A2}). 
\end{itemize}
\subsubsection*{Proof of Lemma \ref{lemma:bd-data}}
For any fixed $i\in[N],j\in[d]$, note that using sub-Gaussian property one has $\mathbb{P}(|X_i(j)|>d^{K_1})\leqslant \exp(-Cd^{2K_1})$, thus $\mathbb{P}(\exists i\in[N],j\in[d]:|X_i(j)|>d^{K_1})\leqslant Nd\exp(-Cd^{2K_1})$, using union bound, which yields the conclusion. 
\subsubsection*{Proof of Lemma \ref{lemma:single-concentration}}
Let 
\[
\mathcal{E}_1\triangleq \left\{\|X_i\|_\infty<d^{K_1},1\leqslant i\leqslant N\right\}.
\]
By Lemma \ref{lemma:bd-data}, $\mathbb{P}(\mathcal{E}_1)\geqslant 1-Nd\exp(-Cd^{2K_1})$. Now, note that
\begin{equation}\label{eq:probbbb}
\mathbb{P}(\mathcal{E}(A)^c) = \mathbb{P}(\mathcal{E}(A)^c|\mathcal{E}_1)\mathbb{P}(\mathcal{E}_1) +\mathbb{P}(\mathcal{E}(A)^c|\mathcal{E}_1^c)\mathbb{P}(\mathcal{E}_1^c)  
\leqslant \mathbb{P}(\mathcal{E}(A)^c|\mathcal{E}_1) + N\exp(-Cd^{2K_1}).
\end{equation}
We now study $\mathbb{P}(\mathcal{E}(A)^c|\mathcal{E}_1)$, hence assume we condition of $\mathcal{E}_1$ from now on. Triangle inequality yields
\[
|Y_i-X_i^T AX_i| \leqslant |X_i^T A X_i | + |X_i^T (W^*)^T W^* X_i|.
\]
Observe now that
\[
\|X_iX_i\|_F^2={\rm trace}(X_iX_i^T X_iX_i^T)=\|X_i\|_2^2 {\rm trace}(X_iX_i^T)=\|X_i\|_2^4,
\]
which implies (conditional on $\mathcal{E}_1$)
\[
\|X_iX_i\|_F =\|X_i\|_2^2\leqslant d^{2K_1+1}.
\]
Now, Cauchy-Schwarz inequality with respect to inner product $\ip{U}{V}\triangleq {\rm trace}(U^T V)$  yields
\begin{align*}
    |X_i^T AX_i| = \ip{A}{X_iX_i^T}\leqslant \|A\|_F\|X_iX_i^T\|_F\leqslant d^{2K_1+2K_2+1},
\end{align*}
for every $i\in[N]$, using $\|A\|_F\leqslant d^{2K_2}$.

Next, let $A^*=(W^*)^T W^*\in\R^{d\times d}$, and let $\eta_1^*,\dots,\eta_d^*$ be the eigenvalues of $A^*$, all non-negative. Observe that
\[
\|W^*\|_F^2 = {\rm trace}(A^*)=\sum_{1\leqslant j\leqslant d}\eta_j^*\leqslant  d^{2K_2}.
\]
Now note that $(\eta_1^*)^2,(\eta_2^*)^2,\dots, (\eta_d^*)^2$ are the eigenvalues of $(A^*)^2 = (A^*)^T A^*$. With this reasoning, we have
\[
\|A^*\|_F^2 ={\rm trace}((A^*)^T A^*)={\rm trace}((A^*)^2) = \sum_{1\leqslant j\leqslant d}(\eta_j^*)^2\leqslant \left(\sum_{1\leqslant j\leqslant d}\eta_j^*\right)^2 \leqslant d^{4K_2}.
\]
Consequently, $\|A^*\|_F\leqslant d^{2K_2}$, and therefore, the exact same reasoning yields
\[
|X_i^T (W^*)^T W^* X_i|=X_i^T A^* X_i\leqslant d^{2K_1+2K_2+1},
\]
for  every $i\in[N]$. 
Hence, conditional on $\mathcal{E}_1$, it holds that for every $i\in[N]$:
\[
\left(X_i^TAX_i - X_i^T (W^*)^T W^* X_i\right)^2 \leqslant 4d^{4K_1+4K_2+2}.
\]
We now apply concentration to i.i.d. sum
\[
\frac1N\sum_{1\leqslant i\leqslant N}\left(X_i^TAX_i - X_i^T (W^*)^T W^* X_i\right)^2
\]
is a sum of bounded random variables that are at most $4d^{4K_1+4K_2+2}$. 

Now, recalling the distributional assumption on the data, we have that conditional on $\|X_i\|_\infty\leqslant d^{K_1}$, the data still has i.i.d. centered coordinates. In particular, the  ``energy barrier" result for the population risk as per Theorem \ref{thm:band-gap} applies: 
\[
\mathbb{E}\left[\left(X^TAX - X^T (W^*)^T W^* X\right)^2
\bigr\vert \mathcal{E}_1\right]\geqslant C_5 \sigma_{\min}(W^*)^4,
\]
where
\[
C_5 = \min\{\mu_4(K_1)-\mu_2(K_1)^2,2\mu_2(K_1)^2\},
\]
is controlled by the conditional moments of data coordinates. 

Finally applying Hoeffding's inequality for bounded random variables we arrive at
$$
\frac1N\sum_{1\leqslant i\leqslant N}\left(X_i^T A X_i - X_i^T (W^*)^T W^* X_i\right)^2\geqslant \frac12 C_5\sigma_{\min}(W^*)^4,
$$
with probability at least $1-\exp\left(-C_3 Nd^{-4K_1-4K_2-2}\right)$. Namely,
$$
\mathbb{P}(\mathcal{E}(A)^c|\mathcal{E}_1)\leqslant \exp(-C_3N d^{-4K_1-4K_2-2}).
$$
Returning to (\ref{eq:probbbb}), this yields
$$
\mathbb{P}(\mathcal{E}_A)\geqslant  1-\exp(-C_3Nd^{-4K_1-4K_2-2}) - Nd\exp(-Cd^{2K_1}),
$$
thus concluding the proof. 
\subsubsection*{Proof of Lemma \ref{lemma:cov-number}}
The proof is almost verbatim from \cite[Lemma~3.1]{candes2011tight},  and included herein for completeness.

Note that any $A\in\R^{d\times d}$, $A\succeq 0$ and  ${\rm rank}(A)=r$ decomposes as $A=Q\Lambda Q^T$, where $Q\in\R^{d\times r}$ satisfying $Q^TQ=I_d$, and $\Lambda\in\R^{r\times r}$, a diagonal matrix with non-negative diagonal entries. Notice, furthermore, that $\|A\|_F=\|\Lambda\|_F\leqslant R$ as $Q$ is orthonormal. With this, we now construct an appropriate net covering the set of all permissible $Q$ and $\Sigma$. 

Let $D$ be the set of all $r\times r$ diagonal matrices with non-negative diagonal entries with Frobenius norm at most $R$. Let $\bar{D}$ be an $\frac{\epsilon}{3}-$net for $D$ in Frobenius norm. Using standard results (see, e.g. \cite[Lemma~5.2]{vershynin2010introduction}), we have 
$$
|\bar{D}|\leqslant \left(\frac{9R}{\epsilon}\right)^r.
$$
Now let $O_{d,r}=\{Q\in\R^{d\times r}:Q^TQ=I_d\}$. To cover $O_{d,r}$ we use a more convenient norm $\|\cdot\|_{1,2}$ defined as
$$
\|X\|_{1,2}=\max_i \|X_i\|_2,
$$
where $X_i$ is the $i^{\rm th}$ column of $X$. Define $Q_{d,r}=\{X\in\R^{d\times r}:\|X\|_{1,2}\leqslant 1\}$. Note that $O_{d,r}\subset Q_{d,r}$. Furthermore, observe also that $Q_{d,r}$ has an $\epsilon-$net of cardinality at most $(3/\epsilon)^{dr}$. With this, we now take $\bar{O}_{d,r}$ to be an $\frac{\epsilon}{3R}-$net for $O_{d,r}$. Consider now the set
$$
\bar{S}_R \triangleq \{\bar{Q}\bar{\Lambda}\bar{Q}^T:\bar{Q}\in\bar{O}_{d,r},\bar{\Lambda}\in\bar{D}\}.
$$
Clearly, 
$$
|\bar{S}_R|\leqslant |\bar{O}_{d,r}||\bar{D}| \leqslant (9R/\epsilon)^{dr+r}.
$$
We now claim $\bar{S}_R$ is indeed an $\epsilon-$net for $S_R$ in Frobenius norm. 
To prove this, take an arbitrary $A\in S_R$, and let $A=Q\Lambda Q^T$. There exists a $\bar{Q}\in \bar{O}_{d,r}$, and a $\bar{\Sigma}\in\bar{D}$ such that $\|\Sigma-\bar{\Sigma}\|_F\leqslant \epsilon/3$, and $\|Q-\bar{Q}\|_{1,2}\leqslant \epsilon/3R$. Now, let $\bar{A}=\bar{Q}\bar{\Sigma}\bar{Q}^T$. Observe that using triangle inequality
\begin{align*}
\|\bar{A}-A\|_F &= \|Q\Lambda Q^T-\bar{Q}\bar{\Lambda}\bar{Q}^T\|_F \\
&\leqslant \|Q\Lambda Q^T-\bar{Q}\Lambda Q^T\|_F +\|\bar{Q}\Lambda Q^T-\bar{Q}\bar{\Lambda}Q^T\|_F+\|\bar{Q}\bar{\Lambda}Q^T-\bar{Q}\bar{\Lambda}\bar{Q}^T\|_F.
\end{align*}
For the first term, notee that since $Q$ is orthonormal, $\|(Q-\bar{Q})\Lambda Q^T\|_F=\|(Q-\bar{Q})\Lambda\|_F$. Next,
$$
\|(Q-\bar{Q})\Lambda\|_F^2 =\sum_{1\leqslant i\leqslant d}\Lambda_{ii}^2 \|Q_i-\bar{Q}_i\|_2^2\leqslant \|Q-\bar{Q}\|_{1,2}^2 \|\Sigma\|_F^2\leqslant (\epsilon/3)^2,
$$
using  $\|Q-\bar{Q}\|_{1,2}\leqslant \epsilon/3R$ and $\|\Sigma\|_F\leqslant R$. Thus, $\|Q\Lambda Q^T-\bar{Q}\Lambda Q^T\|_F\leqslant \epsilon/3$. Similarly, we also have $\|\bar{Q}\bar{\Lambda}Q^T-\bar{Q}\bar{\Lambda}\bar{Q}^T\|_F\leqslant \epsilon/3$. Finally, $\|\bar{Q}\Lambda Q^T-\bar{Q}\bar{\Lambda}Q^T\|_F=\|\Lambda Q^T-\bar{\Lambda}Q^T\|_F=\|\Lambda-\bar{\Lambda}\|_F\leqslant \epsilon/3$ using again the facts that $Q$ and $\bar{Q}$  are both orthonormal. This concludes that $\|\bar{A}-A\|_F\leqslant \epsilon$; thus $|\bar{S}_R|$ is indeed an $\epsilon-$net for $S_R$, in Frobenius norm, of cardinality at most $(9R/\epsilon)^{dr+r}$. 

As a side remark observe that we gain an extra factor of $2$ in the exponent owing to the fact that $A$ is positive semidefinite (otherwise the bound would be $(9R/\epsilon)^{2dr+r}$). 

\subsection{Proof of Theorem \ref{thm:energy-barrier-empirical}}\label{sec:pf-thm:energy-barrier-empirical} 
First, let 
$$
\mathcal{S}_1 \triangleq \left\{W\in\R^{m\times d}:{\rm rank}(W)<d,\risk{W}<\frac12 C_5 \sigma_{\min}(W^*)^4\right\}.
$$
We start with the following claim.
\begin{claim}\label{claim:bounded-norm-W-belowbarrier}
In the setting of Theorem~\ref{thm:energy-barrier-empirical} the following holds. For any $W\in\R^{m\times d}$ with $\risk{W} \leqslant \frac12 C_5\sigma_{\min}(W^*)^4$, it holds that with probability at least $1-2\exp(-C'N)$ for some absolute constant $C'>0$,
$$
\|W\|_F\leqslant d^{K_2+1}.
$$
\end{claim}
\begin{proof}{(of Claim \ref{claim:bounded-norm-W-belowbarrier})}

For convenience, let $\riskk_0\triangleq \frac12 C_5\sigma_{\min}(W^*)^4$, and for the random data vector $X=(X_1,\dots,X_d)\in\R^d$ let $\sigma^2=\mathbb{E}[X_1^2]$. Recall that $X$ has i.i.d. centered coordinates with sub-Gaussian coordinate distribution.

We have the following, where the  implication is due to Cauchy-Schwarz:
$$
\riskk_0\geqslant \frac1N \sum_{1\leqslant i\leqslant N}(Y_i-f(X_i;W))^2 \Rightarrow (\riskk_0)^{1/2}\geqslant \left|\frac1N\sum_{1\leqslant i\leqslant N}(Y_i-f(X_i;W))\right|
$$

We now establish that with probability at least $1-2\exp(-t^2d)$, the following holds, provided $N\geqslant C(t/\epsilon)^2d$: {\bf for every} $W\in\R^{m\times d}$,
$$
\left|\frac1N \sum_{1\leqslant i\leqslant N}X_i^T W^T WX_i - \sigma^2\|W\|_F^2\right|\leqslant \epsilon\sigma^2\|W\|_F^2.
$$
To see this, we begin by  noticing $X_i^T W^T WX_i = {\rm  trace}(X_i^T W^T WX_i)=\ip{W^TW}{X_iX_i^T}$. Using this we have
$$
\left|\frac1N \sum_{1\leqslant i\leqslant N}X_i^T W^T WX_i - \sigma^2\|W\|_F^2\right| = \left|\left\langle W^TW, \frac1N \sum_{1\leqslant i\leqslant N}X_iX_i^T -\sigma^2 I_d\right\rangle\right|.
$$
We now use H\"{o}lder's inequality Theorem \ref{thm:matrix-holder} with $p=1,q=\infty$, $U=W^TW$ and $V=\frac1N\sum_i X_iX_i^T-\sigma^2 I_d$. This yields
$$
\left|\left\langle W^TW, \frac1N \sum_{1\leqslant i\leqslant N}X_iX_i^T -\sigma^2 I_d\right\rangle\right|\leqslant  \|W\|_F^2 \left\|\frac1N \sum_{1\leqslant i\leqslant N}X_iX_i^T -\sigma^2 I_d\right\|.
$$
Observing now $\mathbb{E}[X_iX_i^T]=\sigma^2 I_d$, we have 
$$
 \left\|\frac1N \sum_{1\leqslant i\leqslant N}X_iX_i^T -\sigma^2 I_d\right\|\leqslant \epsilon\sigma^2
$$
with probability at least $1-2\exp(-t^2d)$ provided $N\geqslant C(t/\epsilon)^2 d$, using the concentration result on sample covariance matrix from Vershynin \cite[Corollary~5.50]{vershynin2010introduction}. 
Hence, on this high probability event, the following holds:
$$
\frac1N \sum_{1\leqslant i\leqslant N}X_i^T (W^*)T W^* X_i \leqslant \sigma^2(1+\epsilon)\|W^*\|_F^2 \quad\text{and}\quad 
\frac1N \sum_{1\leqslant i\leqslant N}X_i^T W^T W X_i \geqslant \sigma^2(1-\epsilon)\|W\|_F^2
$$
Hence,
$$
\riskk_0 \geqslant \frac1N\sum_{1\leqslant i\leqslant N}(X_i^T W^T WX_i - X_i^T (W^*)^T W^* X_i)\geqslant \sigma^2(1-\epsilon)\|W\|_F^2 - \sigma^2(1+\epsilon)\|W^*\|_F^2.
$$
This yields, for any $W$ with $\risk{W}\leqslant \riskk_0$, 
$$
\|W\|_F\leqslant \left(\frac{(\riskk_0)^{1/2}}{\sigma^2(1-\epsilon)}+\frac{1+\epsilon}{1-\epsilon}\|W^*\|_F^2\right)^{1/2}
$$
with probability at least $1-2\exp(-t^2d)$. Now, observe that
$$
\sigma_{\min}(W^*)^2 = \lambda_{\min}((W^*)^T W^*) \leqslant {\rm trace}((W^*)^T W^*)\leqslant \|W^*\|_F^2\leqslant d^{2K_2}.
$$
Furthermore, $C_5=O(1)$. This yields
\begin{equation}\label{eq:initial-risk-poly}
\riskk_0=\frac12C_5 \sigma_{\min}(W^*)^4 =O(d^{4K_2}). 
\end{equation}
We now take $\epsilon=1/2$ above, and conclude that 
\begin{align*}
\|W\|_F\leqslant \left(\frac{(\riskk_0)^{1/2}}{\sigma^2(1-\epsilon)}+\frac{1+\epsilon}{1-\epsilon}\|W^*\|_F^2\right)^{1/2} \leqslant  d^{K_2+1}
\end{align*}
for $d$ large enough; with probability at least $1-2\exp(-t^2d)$, which is at least $1-2\exp(-C'N)$ for some constant $C'$ as $N\geqslant C(t/\epsilon)^2d$. 
\end{proof}
Let now
$$
\mathcal{S}_2\triangleq \left\{W\in\R^{m\times d}:{\rm  rank}(W)<d,\risk{W}<\frac12 C_5\sigma_{\min}(W^*)^4,\|W\|_F\leqslant d^{K_2+1}\right\}.
$$
A consequence of Claim~\ref{claim:bounded-norm-W-belowbarrier} is that $\mathbb{P}(\mathcal{S}_1=\mathcal{S}_2)\geqslant 1-2\exp(-C'N)$. We now establish that 
$$
\mathbb{P}(\mathcal{S}_2=\varnothing) \geqslant 1-\left(9d^{4K_1+4K_2+7}\right)^{d^2-1}\left(\exp\left(-C_3 Nd^{-4K_1-4K_2-6}\right)+Nde^{-Cd^{2K_1}}\right),
$$
which, through the union bound, will then yield
$$
\inf_{W\in\R^{m\times d}:{\rm rank}(W)<d}\risk{W}\geqslant \frac12 C_5 \sigma_{\min}(W^*)^4,
$$
with probability at least
$$
1-\exp(-C'N)-\left(9d^{4K_1+4K_2+7}\right)^{d^2-1}\left(\exp\left(-C_3 Nd^{-4K_1-4K_2-6}\right)+Nde^{-Cd^{2K_1}}\right).
$$
Let $A=W^TW\in\R^{d\times d}$. We claim $\|A\|_F\leqslant d^{2K_2+2}$. To see this, note that $\|A\|_F^2 = {\rm trace}(A^T A)={\rm trace}(A^2)$. Let $\theta_1,\dots,\theta_d$ be the eigenvalues of $A$, all non-negative as $A\succeq 0$; and $\theta_1^2,\dots,\theta_d^2$ are the eigenvalues of $A^2$. With this,
$$
{\rm trace}(A^2)=\sum_{1\leqslant i\leqslant d}\theta_i^2 \leqslant \left(\sum_{1\leqslant i\leqslant d}\lambda_i\right)^2 = {\rm trace}(A)^2.
$$
Hence, $\|A\|_F\leqslant {\rm trace}(A)=\|W\|_F^2\leqslant d^{2K_2+2}$, as requested.

Now, let 
$$
S_R =\{A\in\R^{d\times d}:{\rm rank}(A)\leqslant d-1,A\succeq 0,\|A\|_F\leqslant  R\}.
$$
Let $\bar{S}_\epsilon$ be an $\epsilon-$net for $S_{d^{2K_2+2}}$ in  Frobenius norm, where $\epsilon$ to be tuned appropriately later. Using Lemma \ref{lemma:cov-number} we have
$$
|\bar{S}_\epsilon|\leqslant \left(\frac{9d^{2K_2}}{\epsilon}\right)^{d^2-1}.
$$
Using Lemma \ref{lemma:single-concentration}, together with the union bound across the net, it holds that with probability at least 
$$
1-\left(\frac{9d^{2K_2+2}}{\epsilon}\right)^{d^2-1} \left(\exp\left(-C_3\frac{N}{d^{4K_1+4K_2+6}}\right)+Ne^{-Cd^{2K_1}}\right)
$$
it is the case that 
$$
 \frac1N\sum_{1\leqslant i\leqslant N}\left(Y_i - X_i^T A X_i\right)^2 \geqslant \frac12 C_5 \sigma_{\rm min}(W^*)^4,
$$
for every $A\in \bar{S}_\epsilon$, where 
$$
C_5 = \min\{\mu_4(K_1)-\mu_2(K_1)^2,2\mu_2(K_1)^2\}
$$
and $\mu_n(K)=\mathbb{E}[X_i^n||X_i|\leqslant d^{K}]$. 

In the remainder  of the proof, suppose for every $A\in\bar{S}_\epsilon$,
$$
\frac1N\sum_{1\leqslant i\leqslant  N}(Y_i-X_i^T AX_i)^2 \geqslant \frac12 C_5\sigma_{\min}(W^*)^4,
$$
and $\|X_i\|_\infty <d^{K_1}$, $i\in[N]$, which holds with probability at least 
$$
1-\left(\frac{9d^{2K_2+2}}{\epsilon}\right)^{d^2-1} \left(\exp\left(-C_3\frac{N}{d^{4K_1+4K_2+6}}\right)+Ne^{-Cd^{K_1}}\right)-Nd\exp(-Cd^{2K_1}).
$$
Now, let $W\in\R^{m\times d}$ with $\|W\|_F\leqslant d^{K_2+1}$, ${\rm rank}(W)\leqslant d-1$. Let $A=W^T W$ (thus $\|A\|_F\leqslant d^{2K_2+2}$) and $\widehat{A}\in \bar{S}_\epsilon$ be such that $\|A-\widehat{A}\|_F\leqslant \epsilon$. We now estimate
$$
\Delta\triangleq \left|\frac1N \sum_{1\leqslant i\leqslant N}(Y_i-X_i^T AX_i)^2 - \frac1N\sum_{1\leqslant i\leqslant N}(Y_i-X_i^T\widehat{A}X_i)^2\right|.
$$
For notational convenience, let $A^*=(W^*)^T W^*$. Now
\begin{align*}
    \Delta& \leqslant \frac1N \sum_{1\leqslant i\leqslant N}\left|(X_i^T (A-A^*)X_i)^2 - (X_i^T (\widehat{A}-A^*)X_i)^2\right|\\ 
    &=\frac1N\sum_{1\leqslant i\leqslant N}\left|X_i^T (A-\widehat{A})X_i\right|\cdot \left| X_i^T (A+\widehat{A}-2A^*)X_i\right|.
\end{align*}
Now, using Cauchy-Schwarz (for inner product $\ip{M}{N}\triangleq {\rm trace}(M^T N)$)
$$
|X_i^T (A-\widehat{A})X_i| = |\ip{A-\widehat{A}}{X_iX_i^T}|\leqslant \|A-\widehat{A}\|_F \cdot \|X_i\|_2^2,
$$
using $\|X_iX_i^T \|_F = \|X_i\|_2^2$. In particular, we obtain
$$
|X_i^T (A-\widehat{A})X_i |\leqslant \epsilon d^{2K_1+1}.
$$
For the term $|X_i^T (A+\widehat{A}-2A^*)X_i|$, we observe that  triangle inequality yields
$$
\|A+\widehat{A}-2A^*\|_F\leqslant 4d^{2K_2+2}.
$$
Thus
$$
|X_i^T (A+\widehat{A}-2A^*)X_i| \leqslant 4d^{2K_1+2K_2+3}.
$$
Using these, we obtain
$$
\left|\risk{W}-\frac1N\sum_{1\leqslant i\leqslant N}(Y_i-X_i^T \widehat{A}X_i)^2\right|\leqslant 4\epsilon d^{4K_1+2K_2+4}=O(d^{-1})=o_d(1),
$$
taking $\epsilon =d^{-4K_1-2K_2-5}$. Using finally the fact that
$$
\frac1N\sum_{1\leqslant i\leqslant N}(Y_i-X_i^T \widehat{A}X_i)^2 
$$
is bounded away from zero across the net $\bar{S}_\epsilon$, we conclude the proof.
\subsection{Proof of Theorem \ref{thm:band-gap}}\label{sec:proof-of-band-gap}
\begin{itemize}
   \item[(a)] Note first that using Theorem \ref{thm:analytic-exp-pop-risk} part $({\rm c})$, we have:
    $$
    \mathcal{L}(W)\geqslant \min\{{\rm Var}(X_i^2),2\E{X_i^2}^2\}{\rm trace}(A^2).
    $$
Now, fix any $W\in\mathbb{R}^{m\times d}$ with ${\rm rank}(W)<d$. Let $a_1\geqslant \cdots\geqslant a_d$ be the eigenvalues of $(W^*)^T W^*$; $b_1\geqslant\cdots\geqslant b_d$ be the eigenvalues $-W^TW$; and $\lambda_1\geqslant\cdots\geqslant \lambda_d$ be the eigenvalues of $(W^*)^T W^* - W^T W$. Since $W$ is rank-deficient, we have $b_1=0$. Furthermore, $a_d=\sigma_{\min}(W^*)^2$, since the eigenvalues of $(W^*)^T W^*$ are precisely the squares of the singular values of $W^*$. Now, recall the (Courant-Fischer) variational characterization of the eigenvalues \cite{horn2012matrix}. If $M$ is a $d\times d$ matrix with eigenvalues $c_1\geqslant \cdots \geqslant c_d$, then:
$$
c_1 = \max_{x:\|x\|_2=1}x^T M x\quad\text{and}\quad c_d = \min_{x:\|x\|_2=1}x^T M x.
$$
With this, fix an $x\in \R^d$ with $\|x\|_2=1$. Then,
$$
x^T((W^*)^T W^* - W^T W)x \geqslant \min_{x:\|x\|_2=1} x^T(W^*)^T W^*x +x^T(-W^TW)x = a_d + x^T(-W^TW)x.
$$
Since this inequality holds for every $x$ with $\|x\|_2=1$, we can take the max over all $x$, and arrive at,
$$
\lambda_1 = \max_{x:\|x\|_2=1}x^T((W^*)^T W^* - W^T W) x \geqslant a_d+b_1 =a_d \geqslant\sigma_{\min}(W^*)^2.
$$
Now, since $\lambda_1^2,\dots,\lambda_d^2$  are precisely the eigenvalues of $A^2$, we have ${\rm trace}(A^2)= \sum_{i=1}^d \lambda_i^2\geqslant \lambda_1^2$. Hence, for any $W$ with ${\rm rank}(W)<d$, it holds that:
$$
\mathcal{L}(W)\geqslant \min\left\{{\rm Var}(X_i^2),2\E{X_i^2}^2\right\}\lambda_1^2.
$$
Finally, since $\lambda_1^2\geqslant \sigma_{\min}(W^*)^4$, the desired conclusion follows by taking the minimum over all rank-deficient $W$.
\item[(b)] 
Let the eigenvalues of $(W^*)^T W^*$ be denoted by $\lambda_1^*,\dots,\lambda_d^*$, with the corresponding orthogonal eigenvectors $q_1^*,\dots,q_d^*$. Namely,  diagonalize $(W^*)^T W^*$ as $Q^*\Lambda^*(Q^*)^T$ where the columns of $Q^*\in\R^{d\times d}$ are $q_1^*,\dots,q_d^*$, and $\Lambda^*\in\R^{d\times d}$ is a diagonal matrix with $(\Lambda^*)_{i,i}=\lambda_i^*$ for every $1\leqslant i\leqslant d$. Let
$$
\overline{W}=\sum_{j=1}^{d-1}\sqrt{\lambda_j^*}q_j^*(q_j^*)^T \in\R^{d\times d}.
$$
Observe that, $\overline{W}^T \overline{W}=Q^*\overline{\Lambda}Q^*$, where $\overline{\Lambda}\in\R^{d\times d}$ is a diagonal matrix with $(\overline{\Lambda})_{i,i}=(\Lambda^*)_{i,i}$ for every $1\leqslant i\leqslant d-1$, and $(\overline{\Lambda})_{d,d}=0$; and that, ${\rm rank}(\overline{W})=d-1$. Now, let $\overline{W_1},\dots,\overline{W_d}\in\R^d$  be the rows of $\overline{W}$, and fix a $j\in[d]$ such that $\overline{W_j}\neq 0$.  

Having constructed a $\overline{W}\in\R^{d\times d}$, we now prescribe $W\in\R^{m\times d}$ as follows. For $1\leqslant i\leqslant d$, $i\neq j$, let $W_i=\overline{W_i}$, where $W_i$ is the $i^{th}$ row of $W$. Then set $W_j=\frac12 \overline{W_j}$, and for every $d+1\leqslant i\leqslant m$, set $W_i = \frac{\sqrt{3}}{2\sqrt{m-d}}\overline{W_j}$. For this matrix, we now claim
\[
W^T W=\overline{W}^T \overline{W}.
\]   
To see this, fix an $X\in\R^d$, and recall that $X^T W^T W X - X^T \overline{W}^T \overline{W} X= \|WX\|_2^2 - \|\overline{W}X\|_2^2$. We now compute this quantity more explicitly:
\begin{align*}
 \|WX\|_2^2 - \|\overline{W}X\|_2^2&= \sum_{k=1}^d \ip{W_k}{X}^2 - \sum_{k=1}^m \ip{\overline{W}_k}{X}^2 \\
&= \sum_{k=1,k\neq j}^d \ip{W_k}{X}^2 + \ip{W_j}{X}^2 \\
&-\sum_{k=1,k\neq j}^d \ip{W_k}{X}^2  - \ip{\frac12 W_j}{X}^2 - \sum_{k=d+1}^m \ip{\frac{\sqrt{3}}{2\sqrt{m-d}}W_j}{X}^2 \\
&=\ip{W_j}{X}^2  -\frac14\ip{W_j}{X}^2 - \frac{3}{4(m-d)}(m-d) \ip{W_j}{X}^2 = 0.
\end{align*}
Hence, for every $X\in\R^d$, we have:
\[
X^T W^T W X = X^T \overline{W}^T \overline{W} X. 
\]
Now let $\Xi = W^T W - \overline{W}^T \overline{W}$. Note that $\Xi\in\R^{d\times d}$ is symmetric, and $X^T  \Xi X=0$ for every $X\in\R^d$. Now, taking $X$ to be $e_i$, that is, the $i^{th}$ element of the standard basis for the Euclidean space $\R^d$, we deduce $\Xi_{i,i}=0$ for every $i\in[d]$. For the off-diagonal entries, let $X=e_i +e_j$. Then, $X^T \Xi X  = \Xi_{i,i}+\Xi_{i,j}+\Xi_{j,i}+\Xi_{j,j}=0$, which, together with the fact that the diagonal entries of $\Xi$ are zero, imply $\Xi_{i,j}=-\Xi_{j,i}$; namely $\Xi$ is skew-symmetric. Finally, since $\Xi$ is also symmetric we have $\Xi_{i,j}=\Xi_{j,i}$, which then implies for every $i,j\in[d]$, $\Xi_{i,j}=0$, that is, $\Xi=0$, and thus, $W^T W = \overline{W}^T \overline{W}$.

Hence, we have for $W\in\R^{m\times d}$ with ${\rm rank}(W)=d-1$,
$$
W^T W - (W^*)^T W^* = Q^* \Lambda' (Q^*)^T,
$$
with $(\Lambda')_{i,i}=0$ for every $1\leqslant i\leqslant d-1$; and $(\Lambda')_{d,d}=-\lambda_d^*$. Namely, the spectrum of the matrix $A=(W^*)^T W^*-W^T W$ contains only two values: $0$ with multiplicity $d-1$, and $\lambda_d^*$ with multiplicity one. In particular,
$$
{\rm trace}(A) = \lambda_d^* \quad\text{and}\quad {\rm trace}(A^2) = (\lambda_d^*)^2.
$$
Using now the upper bound provided by Theorem (\ref{thm:analytic-exp-pop-risk}) part $({\rm c})$ yields the desired claim. Therefore, the energy band lower bound is tight, up to a  multiplicative constant. 
\end{itemize}

\subsection{Proof of Corollary \ref{coro-1}}\label{sec:proof-of-coro-1}
We do both parts together.
Let $W\in\R^{m\times d}$, and $\widetilde{f}(W;X)= \sum_{j=1}^m \widetilde{\sigma}(\ip{W_j}{X})$ where $\widetilde{\sigma}(x)=\alpha x^2+\beta x+\gamma$. Now,  note the decomposition: $\widetilde{f}(W;X)=\alpha f(W;X)+\beta g(W;X)+\gamma m$, where $f(W;X)=\sum_{j=1}^m \ip{W_j}{X}^2$ and $g(W;X)=\sum_{j=1}^m \ip{W_j}{X}$. 
In particular, defining:
\[
\Delta_f = f(W;X)-f(W^*;X) \quad\text{and}\quad \Delta_g = g(W;X)-g(W^*;X),
\]
we have $(\widetilde{f}(W;X)-\widetilde{f}(W^*;X))^2 = (\alpha\Delta_f+\beta\Delta_g)^2\geqslant \alpha^2\Delta_f^2 +2\alpha\beta\Delta_f \Delta_g$. Taking expectations on both sides with respect to $X$, we then have for
$$\widetilde{\mathcal{L}}(W)\triangleq \mathbb{E}[(\widetilde{f}(W;X)-\widetilde{f}(W^*;X))^2],
$$
it is the case that
\begin{equation}\label{eq:quad-lower}
\widetilde{\mathcal{L}}(W)\geqslant \alpha^2\mathcal{L}(W) + 2\alpha\beta\E{\Delta_f\Delta_g}=\alpha^2\mathcal{L}(W) +2\alpha\beta\sum_{1\leqslant i,j,k\leqslant d}\E{X_i X_jX_k}A_{ij}\theta_k =\alpha^2\mathcal{L}(W),
\end{equation}
where $\mathcal{L}(W)=\mathbb{E}[(f(W;X)-f(W^*;X))^2]$, when $f(W;X)=\sum_{1\leqslant j\leqslant m}\ip{W_j}{X}^2$, $A=W^T W-(W^*)^T W^*$, and $\theta_k = \sum_{j=1}^m W_{j,k}-W^*_{j,k}$. Taking the minimum over all rank deficient matrices in Equation (\ref{eq:quad-lower}), we arrive at:
\begin{align*}
 \min_{W\in\mathbb{R}^{m\times d}:{\rm rank}(W)<d} \widetilde{\mathcal{L}}(W)&\geqslant \alpha^2 \min_{W\in\mathbb{R}^{m\times d}:{\rm rank}(W)<d} \mathcal{L}(W)\\
 &\geqslant \alpha^2 \min\{{\rm  Var}(X_i^2),2\E{X_i^2}^2\}\cdot \sigma_{\min}(W^*)^4,
\end{align*} 
where the second inequality is due to Theorem \ref{thm:band-gap} $({\rm a})$. This concludes the population version.  
With this, the extension to empirical risk is immediate, by inspecting the proof of Theorem~\ref{thm:energy-barrier-empirical}. 
\subsection{Proof of Theorem \ref{thm:full-rank-empirical-global-opt} }\label{sec:pf-thm:full-rank-empirical-global-opt}
We start by computing $\nabla \risk{W}$. Taking  derivatives with respect to $j^{th}$ row $W_j$ of $W\in\R^{m\times d}$,  we arrive at
$$
\nabla_{W_j}\risk{W} = \frac4N \sum_{1\leq i \leq N}\left(\sum_{1\leq j\leq m}\ip{W_j}{X_i}^2-Y_i\right)\ip{W_j}{X_i}X_i.
$$
Interpreting these gradients as a row vector and aggregating into a matrix, we then have
$$
\nabla_W \risk{W} = W\left(\frac4N\sum_{1\leq i \leq N}\left(\sum_{1\leq j\leq m}\ip{W_j}{X_i}^2-Y_i\right)X_iX_i^T\right).
$$
Assume now that ${\rm rank}(W)=d$, and $\nabla \risk{W}=0$. We then arrive at
$$
\frac1N \sum_{1\leq i \leq N}\left(\sum_{1\leq j\leq m}\ip{W_j}{X_i}^2-Y_i\right)X_iX_i^T = 0.
$$
We now claim that $\risk{W}=0$. To see this, we take a route  similar to \cite[Lemma~6.1]{soltanolkotabi2018theoretical}.

Let $M\triangleq W^TW$, and consider  the function
$$
f(M)\triangleq \frac1N \sum_{1\leq i\leq N}(Y_i-X_i^T MX_i)^2.
$$
Observe that $f(\cdot)$ is quadratic in $M$. Thus, any $\widehat{M}$ with $\nabla f(\widehat{M})=0$, that is
$$
\frac1N\sum_{1\leq i\leq N}(X_i^T \widehat{M} X_i-Y_i)X_iX_i^T=0
$$
it is the case that $\widehat{M}$ is a global optimum of $f$. In particular for any $M\in \R^{d\times d}$,  $f(M)\geqslant f(\widehat{M})$. Now, take any $\bar{W}\in\R^{m\times d}$, and observe that $\risk{\bar{W}}=f(\bar{W}^T \bar{W})$. Since  $\nabla  f(W^T W)=0$, it follows that
$$
\risk{\bar{W}} = f(\bar{W}^T \bar{W})\geqslant  f(W^T W)=\risk{W}.
$$
Namely, $W$ is indeed a global optimizer of $\risk{\cdot}$. Since $W=W^*$ makes the cost zero, we  obtain $\risk{W}=0$.

Now, using Theorem \ref{thm:random-data-geo-cond}, we obtain that ${\rm span}(X_iX_i^T:1\leqslant i\leqslant N)$ is the set of all $d\times d$ symmetric matrices; with probability one, provided $N\geqslant d(d+1)/2$. In this case, using Theorem \ref{thm:geo-condition}, we conclude that $W^T W=(W^*)^T W^*$, concluding the proof. 


\subsection{Proof of Theorem \ref{thm:full-rank-global-opt}}\label{sec:proof-of-full-rank-global-opt}
We first establish the following proposition, for any $W$, which is a stationary point of the population risk.
\begin{proposition}\label{prop-stat-eq} Let $\mathcal{D}^*\in\R^{d\times d}$ be a diagonal matrix with $\mathcal{D}^*_{ii}=((W^*)^T W^*)_{ii}$, and define $\mathcal{D}\in\R^{d\times d}$ analogously. Then,   $W\in\R^{m\times d}$ enjoys the ``stationarity equation":
\begin{align*}
&(\mu_4-3\mu_2^2)W\mathcal{D}^* + \mu_2^2 W\|W^*\|_F^2 +2\mu_2^2(W(W^*)^T W^*)\\
&=(\mu_4-3\mu_2^2)W\mathcal{D} +\mu_2^2 W\|W\|_F^2 +2\mu_2^2(W(W^TW)).
\end{align*}
\end{proposition}
\begin{proof}
To that end, fix a $k_0\in[m]$ and $\ell_0\in [d]$. Note that, $\nabla_{k_0,\ell_0}\mathcal{L}(W)=\E{\nabla_{k_0,\ell_0}(f(W^*;X)-f(W;X))^2}$, using dominated convergence theorem. Next, $\E{\nabla_{k_0,\ell_0}(f(W^*;X)-f(W;X))^2}=0$ implies that, for every $k_0\in [m]$ and $\ell_0\in[d]$:
    $$
    \sum_{j=1}^m \E{\ip{W_j^*}{X}^2 \ip{W_{k_0}}{X}X_{\ell_0}} =  \sum_{j=1}^m \E{\ip{W_j}{X}^2 \ip{W_{k_0}}{X}X_{\ell_0}}.
    $$
 Note next that, $ \sum_{j=1}^m \E{\ip{W_j^*}{X}^2 \ip{W_{k_0}}{X}X_{\ell_0}}$ computes as,
    \begin{align*}
   & \mu_4\sum_{j=1}^m (W_{j,\ell_0}^*)^2 W_{k_0,\ell_0} + \mu_2^2\sum_{j=1}^m \sum_{1\leqslant\ell\leqslant d,\ell\neq \ell_0}W_{k_0,\ell_0}(W_{j,\ell}^*)^2 +2\mu_2^2\sum_{j=1}^m \sum_{1\leqslant\ell\leqslant d,\ell\neq \ell_0}W_{k_0,\ell}W_{j,\ell}^* W_{j,\ell_0}^*.
   \end{align*}
   We now put this object into a more convenient form. Notice that the expression above is
   $$
   (\mu_4-3\mu_2^2) A_{k_0,\ell_0}+ \mu_2^2 B_{k_0,\ell_0} + 2\mu_2^2 C_{k_0,\ell_0},
   $$
   where
   $$
   A_{k_0,\ell_0} = W_{k_0,\ell_0}\sum_{j=1}^m (W_{j,\ell_0}^*)^2 \quad\text{and}\quad B_{k_0,\ell_0}=\sum_{j=1}^m \sum_{\ell=1}^d W_{k_0,\ell_0}(W_{j,\ell}^*)^2 \quad\text{and}\quad C_{k_0,\ell_0}= \sum_{j=1}^m \sum_{\ell=1}^d W_{k_0,\ell} W_{j,\ell_0}^* W_{j,\ell}^*.
   $$
   Observe that, $B_{k_0,\ell_0} = W_{k_0,\ell_0}\|W^*\|_F^2$. We now study $A_{k_0,\ell_0}$ and  $C_{k_0,\ell_0}$ more carefully. Observe that $\sum_{j=1}^m (W_{j,\ell_0}^*)^2 = ((W^*)^T W^*)_{\ell_0,\ell_0}$. Now, let $\mathcal{D}^*\in\R{d\times d}$ be a diagonal matrix where $(\mathcal{D}^*)_{ij} = ((W^*)^T W^*)_{ii}$, if $i=j$; and $0$ otherwise. We then have $A_{k_0,\ell_0}=(W\mathcal{D}^*)_{k_0,\ell_0}$. 
   We now study $C_{k_0,\ell_0}$. Recall that $W_i^*$ is the $i^{th}$ row $W^*$.  Observe that, $\sum_{j=1}^m W_{j,\ell_0}^* W_{j,\ell}^*=((W^*)^TW^*)_{\ell_0,\ell}$. Hence, 
   $$
   \sum_{j=1}^m \sum_{\ell=1}^d W_{k_0,\ell}W_{j,\ell_0}^*W_{j,\ell}^*=\sum_{\ell=1}^d \sum_{j=1}^m W_{k_0,\ell}W_{j,\ell_0}^*W_{j,\ell}^*=\sum_{\ell=1}^d W_{k_0,\ell}((W^*)^TW^*)_{\ell_0,\ell}=(W((W^*)^T W^*))_{k_0,\ell_0},
   $$
   that is, $C_{k_0,\ell_0}=(W((W^*)^T W^*))_{k_0,\ell_0}$. Combining everything, we have that for every $k_0\in[m]$ and $\ell_0\in[d]$:
   $$
   \sum_{j=1}^m \E{\ip{W_j^*}{X}^2 \ip{W_{k_0}}{X}X_{\ell_0}} = (\mu_4-3\mu_2^2)(W\mathcal{D}^*)_{k_0,\ell_0} + \mu_2^2 W_{k_0,\ell_0}\|W^*\|_F^2 + 2\mu_2^2(W((W^*)^T W^*))_{k_0,\ell_0}.
   $$
    In particular, stationarity yields:
    \begin{equation}\label{eq:stationarityy}
    (\mu_4-3\mu_2^2)W\mathcal{D}^* + \mu_2^2 W\|W^*\|_F^2 + 2\mu_2^2(W((W^*)^T W^*)) = 
    (\mu_4-3\mu_2^2)W\mathcal{D} + \mu_2^2 W\|W\|_F^2 + 2\mu_2^2 W(W^TW),
   \end{equation}
   where the $d\times d$ diagonal matrix $\mathcal{D}$ is defined as $\mathcal{D}_{ii}=(W^T W)_{ii}$; and entrywise equalities are converted into equality of two matrices by varying $k_0\in[m]$ and $\ell_0\in[d]$. 
  \end{proof}
  Having now established the Proposition \ref{prop-stat-eq}  for the ''stationarity equation", we now study its implications for any full-rank $W$.

    Let $W\in \R^{m\times d}$ be a stationary point with ${\rm rank}(W)=d$. We first establish  $\|W\|_F=\|W^*\|_F$. Since $W\in \R^{m\times d}$ is a stationary point, it holds that for every $(k_0,\ell_0)\in[m]\times [d]$, $\nabla_{k_0,\ell_0}\mathcal{L}(W) = 0$. In particular, Equation (\ref{eq:stationarityy}) holds. 
    
        Recalling now that $W$ is full rank, it follows from the rank-nullity theorem that ${\rm ker}(W)$ is trivial, that is, ${\rm ker}(W)=\{0\}$. Hence, for matrices $M_1,M_2$ (with matching dimensions), whenever $WM_1=WM_2$ holds, we deduce $M_1=M_2$, since each column of $M_1-M_2$ is contained in ${\rm ker}(W)$. Thus, Equation (\ref{eq:stationarityy}) then yields:
    \begin{equation}\label{eq:take-trace-here}
    (\mu_4-3\mu_2^2)\mathcal{D}^* + \mu_2^2 \|W^*\|_F^2 I_d + 2\mu_2^2(W^*)^T W^*  = 
    (\mu_4-3\mu_2^2)\mathcal{D} + \mu_2^2 \|W\|_F^2 I_d + 2\mu_2^2 W^T W.
    \end{equation} 
    Next, note that ${\rm trace}(\mathcal{D^*})=\sum_{i=1}^d ((W^*)^T W^*)_{ii} = {\rm trace}((W^*)^T W^*) = \|W^*\|_F^2$, and similarly, ${\rm trace}(\mathcal{D}) = \|W\|_F^2$. In particular, taking traces of both sides in Equation (\ref{eq:take-trace-here}), we get
    $$
    (\mu_4-\mu_2^2)\|W^*\|_F^2 + \mu_2^2 d\|W^*\|_F^2= 
    (\mu_4-\mu_2^2)\|W\|_F^2 + \mu_2^2 d\|W\|_F^2,
    $$
    implying that $\|W^*\|_F^2=\|W\|_F^2$. Incorporating this into Equation (\ref{eq:take-trace-here}), we then arrive at:
    $$
    (\mu_4-3\mu_2^2)\mathcal{D}^* +2\mu_2^2 (W^*)^T W^* = 
    (\mu_4-3\mu_2^2)\mathcal{D} +2\mu_2^2 W^T W.
    $$
    Now, suppose $i\in[d]$. Note that inspecting $(i,i)$ coordinate above, we get:
    $$
    (\mu_4-3\mu_2^2)((W^*)^T W^*)_{ii} + 2\mu_2^2 ((W^*)^T W^*)_{ii}=
    (\mu_4-3\mu_2^2)(W^T W)_{ii} + 2\mu_2^2 (W^T W)_{ii}.
    $$
    Since $\mu_4-\mu_2^2={\rm Var}(X_i^2)>0$, we then get
    $$
    ((W^*)^T W^*)_{ii}=(W^T W)_{ii}.
    $$
    Now, focus on off-diagonal entries, by fixing $i\neq j$. Observe that since ${\rm Var}(X_i^2)>0$, it also holds $\E{X_i^2}=\mu_2>0$. Now note that, $\mathcal{D}^*_{ij} = \mathcal{D}_{ij}=0$ in this case. We then have,
    $$
    2\mu_2 ((W^*)^T W^*)_{ij} =2\mu_2(W^T W)_{ij}\Rightarrow (W^*)^T W^* = W^T W.
    $$
We conclude that the matrix $(W^*)^T W^*-W^TW$ is a zero matrix. Hence, $W=QW^*$ for some orthonormal $Q\in\R^{m\times m}$ per Theorem \ref{thm:auxiliary-2}, and $\mathcal{L}(W)=0$.

   \subsection{Proof of Theorem \ref{thm:gd-conv-empirical}}\label{sec:pf-thm:gd-conv-empirical}
   \subsubsection*{Part ${\rm (a)}$}
Note that by Claim~\ref{claim:bounded-norm-W-belowbarrier}, it follows that with probability at least $1-2\exp(-C'N)$, it is the case that for any $W$ with $\risk{W}\leqslant \risk{W_0}<\frac12 C_5\sigma_{\min}(W^*)^4$, $\|W\|_F\leqslant d^{K_2+1}$. 
Now let 
\begin{equation}\label{eq:upper-bd-normW-belowbarrier}
\mathcal{E}_1\triangleq \left\{\sup_{W:\risk{W}\leqslant \riskk_0}\|W\|_F\leqslant d^{K_2+1}\right\}
\end{equation}
thus $\mathbb{P}(\mathcal{E}_1)\geqslant 1-2\exp(-C'N)$ and
\begin{equation}\label{eq:highprobdatabd}
    \mathcal{E}_2\triangleq  \left\{\|X_i\|<d^{K_1},1\leqslant i\leqslant N\right\},
\end{equation}
such that $\mathbb{P}(\mathcal{E}_2)\geqslant 1-Nd\exp(-Cd^{2K_1})$ as per Lemma~\ref{lemma:bd-data}. 

Note that the $\|\nabla^2\risk{W}\| = {\rm poly}(\|W\|_F,\|X_1\|,\dots,\|X_N\|)$. Thus on the event $\mathcal{E}_1\cap \mathcal{E}_2$, which holds with probability at least $1-Nd\exp(-Cd^{2K_1})-2\exp(-C'N)$, we have that 
$$
L=\sup\left\{\|\nabla^2\risk{W}\|:\risk{W}\leqslant \riskk_0\right\}={\rm poly}(d)<+\infty
$$
as claimed. 

\subsubsection*{Part {\rm (b)}}
Suppose that the event $\mathcal{E}_1\cap\mathcal{E}_2$ (defined above in \eqref{eq:upper-bd-normW-belowbarrier} and \eqref{eq:highprobdatabd} takes place. 

We now run the gradient descent with a step size of $\eta<1/2L$: a second order Taylor expansion reveals that
$$
\risk{W_1}-\risk{W_0}\leqslant -\eta \|\nabla \risk{W_0}\|_F^2/2
$$
where $\nabla \risk{W}$ is the gradient of the empirical risk evaluated at $W$. In particular, $\risk{W_1}\leqslant \risk{W_0}$. Since $\mathcal{E}_1$ takes place, we conclude $\|\nabla^2 \risk{W_1}\|\leqslant L={\rm poly}(d)$, where $\|\nabla^2 \risk{W}\|$ is the spectral norm of the Hessian matrix $\nabla^2 \risk{W}$. From here, we induct on $k$: induction argument reveals that we can retain a step size of $\eta<1/2L$ (thus $\eta={\rm poly}(d)$), and furthermore along the trajectory $\{W_k\}_{k\geqslant 0}$, it holds:
$$
\risk{W_{k+1}} - \risk{W_k}\leqslant -\eta \|\nabla \risk{W_k}\|_F^2/2.
$$
Now let $T$ be the first time for which $\|\nabla \risk{W}\|_F\leqslant \epsilon$, namely the horizon required to arrive at an $\epsilon-$stationary point. We claim $T={\rm poly}(\epsilon^{-1},d,\sigma_{\min}(W^*)^{-1})$. 

To see this, note that from the definition of $T$, it holds that $\|\nabla \risk{W_t}\|_F\geqslant \epsilon$ as $t\leqslant T-1$. Now, a telescoping argument together with $\eta=1/{\rm poly}(d)$ reveals
$$
\risk{W_T} - \risk{W_0}\leqslant -T({\rm poly}(d))^{-1}\epsilon^2.
$$
Using now $\risk{W_T}\geqslant 0$, we conclude $\risk{W_0}\geqslant T\epsilon^2 {\rm poly}(d)$. Since $\risk{W_0}=\riskk_0$ is at most polynomial in $d$ as per \eqref{eq:initial-risk-poly}, we conclude $T={\rm poly}(\epsilon^{-1},d)$. 

We now turn our attention bounding its risk. Let $r_i\triangleq Y_i-X_i^T W^T WX_i$. Note that $\risk{W}=\frac1N\sum_{1\leqslant i\leqslant N}r_i^2$. Now,
    \begin{align*}
        \risk{W}&=\frac1N\sum_{1\leqslant i\leqslant N}r_i (X_i^T (W^*)^T W^* X_i-X_i^T W^T WX_i)\\
        &=\ip{W^TW-(W^*)^TW^*}{\frac1N\sum_{1\leqslant i\leqslant N}r_iX_iX_i^T}.
    \end{align*}
    Using Cauchy-Schwarz inequality, we have
    \begin{align*}
    \risk{W} &= \left|\ip{ W^TW-(W^*)^TW^*}{\frac1N\sum_{1\leqslant i\leqslant N}r_iX_iX_i^T}\right|\\
    &\leqslant \|W^TW-(W^*)^T W^*\|_F \cdot \left\|\frac1N\sum_{1\leqslant i\leqslant N}r_iX_iX_i^T\right\|_F.
        \end{align*}
    Next, $\|W^TW\|_F^2 = {\rm trace}((W^TW)^2)\leqslant ({\rm trace}(W^TW))^2=\|W\|_F^4$, using the fact that $W^TW\succeq 0$. In particular, on the event $\mathcal{E}_1$ defined as per \eqref{eq:upper-bd-normW-belowbarrier}, we conclude that $\|W\|_F\leqslant d^{K_2+1}$, and therefore $\|W^TW\|_F\leqslant d^{2K_2+2}$. This, together with $\|W^*\|_F\leqslant d^{K_2}$ and triangle inequality then yields
    \[
    \|W^TW-(W^*)^TW^*\|_F\leqslant 2d^{2K_2+2},
    \]
    with probability at least $1-\exp(-CN)$. 
    Hence, on this event
    \begin{equation}\label{eq:bound-risik}
        \risk{W}\leqslant 2d^{2K_2+2}\left\|\frac1N\sum_{1\leqslant i\leqslant N}r_iX_iX_i^T\right\|_F.
    \end{equation}
    With this, we now turn our attention to bounding
    \[
    \left\|\frac1N\sum_{1\leqslant i\leqslant N}r_iX_iX_i^T\right\|_F.
    \]
    We claim that for the event
    \begin{equation}\label{eq:event-3}
    \mathcal{E}_3\triangleq \left\{\inf_{\substack{W\in\R^{m\times d}:\sigma_{\min}(W)<\frac12\sigma_{\min}(W^*)\\ \|W\|_F\leqslant d^{K_2+1}}} \risk{W} \geqslant \frac12C_5 \sigma_{\min}(W^*)^4\right\},
    \end{equation}
    it is the case that
    \begin{equation}\label{eq:event-3sprob}
    1-(9d^{4K_1+4K_2+3})^{d^2-1}\left(\exp(-C_4 Nd^{-4K_1-4K_2-2}) + Nd\exp(-Cd^{2K_1})\right).
        \end{equation}
    This is almost a straightforward modification of the proof of earlier energy barrier result Theorem \ref{thm:energy-barrier-empirical}, and we only point out required modifications. Take any $W\in\R^{m\times d}$ with $\sigma_{\min}(W)<\frac12\sigma_{\min}(W^*)$. In particular,
    \[
    \lambda_{\min}(W^TW) = \sigma_{\min}(W)^2<\frac14\sigma_{\min}(W^*)^2.
    \]
    Inspecting now the proof of Theorem~\ref{thm:band-gap}{\rm (a)}, we obtain that for such a $W$, 
    \[
    \mathbb{E}\left[(X^TW^TWX-X^T(W^*)^TW^*X)^2\big\vert \|X\|_\infty<d^{K_1}\right]\geqslant \frac34 C_5 \sigma_{\min}(W^*)^4,
    \]
    and consequently, modifying Lemma \ref{lemma:single-concentration}, we have that
    \[
    \mathbb{P}\left(\frac1N\sum_{1\leqslant i\leqslant N}(Y_i-X_i^TW^TWX_i)^2 \geqslant  \frac12 C_5\sigma_{\min}(W^*)^4\right)\geqslant   1-\exp(-C'Nd^{-4K_1-4K_2-2})-Nd\exp(-Cd^{2K_1}).
    \]
    Using now a covering numbers bound, in an exact same manner as in the proof of Theorem~\ref{thm:energy-barrier-empirical}, we conclude that
    \[
    \inf_{\substack{W\in\R^{m\times d}:\sigma_{\min}(W)<\frac12\sigma_{\min}(W^*)\\ \|W\|_F\leqslant d^{K_2+1}}} \risk{W} \geqslant \frac12C_5 \sigma_{\min}(W^*)^4
    \]
    with probability at least
    \[
    1-(9d^{4K_1+4K_2+3})^{d^2-1}\left( \exp(-C_4 Nd^{-4K_1-4K_2-2}) +(9d^{4K_1+4K_2+3})^{d^2-1} Nd\exp(-Cd^{2K_1})\right).
    \]
    Now suppose in the remainder of this part that the event $\mathcal{E}_1\cap \mathcal{E}_2\cap \mathcal{E}_3$ which is 
    \begin{align*}
       &  \left\{\sup_{W:\risk{W}\leqslant \riskk_0}\|W\|_F\leqslant d^{K_2+1}\right\} \bigcap \{\|X_i\|_\infty<d^{K_1},1\leqslant i\leqslant N\} \\
&\bigcap \left\{\inf_{\substack{W\in\R^{m\times d}:\sigma_{\min}(W)<\frac12\sigma_{\min}(W^*)\\ \|W\|_F\leqslant d^{K_2+1}}} \risk{W} \geqslant \frac12C_5 \sigma_{\min}(W^*)^4\right\}
    \end{align*}
    holds true. In particular, for any $W$ with risk less than $\frac12C_5\sigma_{\min}(W^*)^4$, we have $\sigma_{\min}(W)>\frac12\sigma_{\min}(W^*)$. Now, take the $\epsilon-$stationary point $W$ generated by the gradient descent. Due to the event $\mathcal{E}_3$, and the fact $\risk{W}<\riskk_0$ proven earlier; it holds that ${\rm rank}(W)=d$, and from the definition of $\epsilon-$stationarity, we have
   $$
   \| \nabla\risk{W}\|_F\leqslant \epsilon.
   $$
   Inspecting the proof of Theorem~\ref{thm:full-rank-empirical-global-opt}, we observe that
   $$
   \nabla \risk{W} = 4W\left(\frac1N\sum_{1\leqslant i\leqslant N}r_iX_iX_i^T\right).
   $$
   Thus we arrive at
   $$
   \left\|W\left(\frac1N\sum_{1\leqslant i\leqslant N}r_iX_iX_i^T\right)\right\|_F\leqslant 4\epsilon.
   $$
   Let
   $$
   B\triangleq W\left(\frac1N\sum_{1\leqslant i\leqslant N}r_iX_iX_i^T\right).
   $$
   Note now that
   $$
   \frac1N\sum_{1\leqslant i\leqslant N}r_iX_iX_i^T = (W^TW)^{-1}W^TB.
   $$
   Next, we have
   $$
   \|(W^TW)^{-1}\|_2 =\frac{1}{\sigma_{\min}(W^TW)} = \frac{1}{\sigma_{\min}(W)^2}<\frac{4}{\sigma_{\min}(W^*)^2},
   $$
   due to conditioning on $\mathcal{E}_3$ \eqref{eq:event-3} above. Furthermore,
   $$
   \|W^T\|_2 = \|W\|_2 = \sqrt{\lambda_{\max}(W^TW)}\leqslant \sqrt{{\rm trace}(W^TW)}=\|W\|_F\leqslant d^{K_2+1}.
   $$
   We now combine these finding.
   \begin{align*}
      \left\|\frac1N\sum_{1\leqslant i\leqslant N}r_iX_iX_i^T\right\|_F&=\|(W^TW)^{-1}W^T B\|_F\\
     & \leqslant \|(W^TW)^{-1}\|_2 \|W^T B\|_F\\
     & \leqslant \|(W^TW)^{-1}\|_2 \|W^T\|_2  \|B\|_F\\
     &\leqslant 16\epsilon \sigma_{\min}(W^*)^{-2}d^{K_2+1}.
   \end{align*}

 We now use the bounds on $\mathbb{P}(\mathcal{E}_1)$ as per \eqref{eq:upper-bd-normW-belowbarrier}, on $\mathbb{P}(\mathcal{E}_2)$ as per \eqref{eq:highprobdatabd}, and on $\mathbb{P}(\mathcal{E}_3)$ as per \eqref{eq:event-3sprob}; to control $\mathbb{P}(\mathcal{E}_1\cap \mathcal{E}_2\cap\mathcal{E}_3)$. We conclude by the union bound that with probability at least
      $$
   1-\exp(-CN)-(9d^{4K_1+4K_2+3})^{d^2-1}\left(\exp(-C_4 Nd^{-4K_1-4K_2-2}) + Nd\exp(-Cd^{2K_1})\right),
   $$
   it holds that for any $W$ with $\|\nabla \risk{W}\|_F\leqslant \epsilon$, its empirical risk is controlled as per \eqref{eq:bound-risik}:
   $$
   \risk{W}\leqslant 32\epsilon\sigma_{\min}(W^*)^{-2}d^{4K_2+4}.
   $$
\subsubsection*{Part {\rm (c)}}

Let $W\in\R^{m\times d}$ be such that $\risk{W}\leqslant  \kappa$. Define the matrix 
$$
M\triangleq W^TW-(W^*)^TW^*.
$$
We will bound $\|M\|_F$, which will ensure weights $W^TW$ are uniformly close to ground truth weights defined $(W^*)^TW^*$. We start by conditioning: assume in the remainder that the event $\mathcal{E}_2$ in \eqref{eq:highprobdatabd} stating $\|X_i\|_\infty<d^{K_1}$, for every $i\in[N]$ is true: this holds with probability at least $1-Nd\exp(-Cd^{2K_1})$, as per Lemma~\ref{lemma:bd-data}. 

Note that
$$
\risk{W} =  \frac1N\sum_{1\leqslant i\leqslant N}(X_i^TMX_i)^2.
$$
To this end, consider a matrix $\Xi\in\R^{N\times d(d+1)/2}$, consisting of i.i.d. rows where $i^{\rm th}$ row of $\Xi$ is $\mathcal{R}_i\triangleq (X_i(1)^2,\dots,X_i(d)^2,X_i(k)X_i(\ell):1\leqslant k<\ell\leqslant d)\in\R^{d(d+1)/2}$. Next, let $$
\Sigma = \mathbb{E}[\mathcal{R}_i\mathcal{R}_i^T]\in\R^{\frac{d(d+1)}{2}\times \frac{d(d+1)}{2}},
$$
where $\mathcal{R}_i$ is the $i^{\rm th}$ row of matrix $\Xi$. 
 Furthermore, let $\mathcal{M}\in\R^{d(d+1)/2}$ be a vector consisting of entries $M_{11},\dots,M_{dd}$; and $2M_{ij}$, $1\leqslant i<j\leqslant  d$. With this notation, if $v=\Xi \mathcal{M}\in\R^{N\times 1}$, then we have
$$
\risk{W} = \|v\|_2^2/N\Rightarrow \|v\|_2^2 \leqslant N\kappa.
$$
Now we have
\begin{equation}\label{eq:salca}
\mathcal{M} = (\Xi^T \Xi)^{-1}\Xi^T v\Rightarrow \|\mathcal{M}\|_2^2\leqslant \|(\Xi^T\Xi)^{-1}\|_2^2 \|\Xi^Tv\|_2^2.
\end{equation}
We start with the second term. Recall that $\|v\|_2\leqslant \sqrt{N\kappa}$, and we condition on $\|X_i\|_\infty<d^{K_1}$, $1\leqslant i\leqslant N$. Next, $|(\Xi^Tv)_i|\leqslant \|v\|_2 \sqrt{Nd^{2K_1}}\leqslant Nd^{K_1}\sqrt{\kappa}$. Hence,
\begin{equation}\label{eq:salca-1}
    \|\Xi^Tv\|_2^2\leqslant N^2 d^{2K_1+2}\kappa.
\end{equation}
We now control $\|(\Xi^T\Xi)^{-1}\|_2^2$. This is done in a manner similar to the proof of \cite[Theorem~3.2]{emschwiller2020neural}. The main tool is the result Theorem \ref{thm:vershy} for concentration
of the spectrum of random matrices with i.i.d. non-isotropic rows.
The parameter setting we operate under is provided below.
  \begin{center}
 \begin{tabular}{||c c||} 
 \hline
 Parameter  & Value \\ [0.5ex] 
 \hline\hline
 $m$  & $d^{2K_1+1}$ \\
 \hline
 $t$  & $N^{1/8}$ \\
 \hline
 $\delta$ & $N^{-3/8}d^{K_1+\frac12}$ \\
 \hline
 $\gamma$ & $\max(\|\Sigma\|^{1/2}\delta,\delta^2)$
 \\
 \hline 
\end{tabular}
\end{center}
Start by verifying  that since we condition on $\|X_i\|_\infty<d^{K_1}$, it is indeed the case that $\ell_2-$norm of each row of $\Xi$ is at most $d^{K_1+\frac12}$, thus the value of $m$ above works.

We now claim $\gamma = \|\Sigma\|^{1/2}\delta$. To prove this it suffices to show $$
N>\|\Sigma\|^{-4/3}d^{\frac{8K_1}{3}+\frac{4}{3}}.
$$
Using \cite[Theorem~5.1]{emschwiller2020neural} with $k=2$, we obtain $\sigma_{\min}(\Sigma)\geqslant cd^{-4}$, for some absolute constant $c>0$ depending only on the data  coordinate distribution. Consequently,
$$
\|\Sigma\|^{-4/3}\leqslant  \sigma_{\min}(\Sigma)^{-4/3}\leqslant c^{-4/3}d^{16/3}\Rightarrow \|\Sigma\|^{-4/3}d^{\frac{8K_1}{3}+\frac43} <c^{-4/3}d^{\frac{20}{3} + \frac{8K_1}{3}},
$$
which is below sample size $N$, as requested. Therefore, $\gamma = \|\Sigma\|^{1/2}\delta$. 

We now claim
$$
\frac12\sigma_{\min}(\Sigma)>\gamma = \|\Sigma\|^{1/2}N^{-\frac38}d^{K_1+\frac12}.
$$
This is equivalent to establishing
$$
N>2^{8/3}\frac{\|\Sigma\|^{4/3}d^{\frac{8K_1}{3}+\frac43}}{\sigma_{\min}(\Sigma)^{8/3}}.
$$
Using again \cite[Theorem~5.1]{emschwiller2020neural}, we have $\|\Sigma\|<fd^4$ for some absolute constant $f>0$. This yields 
$$
2^{8/3}\frac{\|\Sigma\|^{4/3}d^{\frac{8K_1}{3}+\frac43}}{\sigma_{\min}(\Sigma)^{8/3}}<C'd^{\frac{52}{3}+\frac{8K_1}{3}}
$$
for some absolute constant $C'>0$, which again holds for our case as $N>d^{18+\frac{8K_1}{3}}$. 

The rest is verbatim from \cite[p45]{emschwiller2020neural}: we now apply Theorem \ref{thm:vershy}. With probability at least $1-d^{2K_1+1}\exp(-cN^{1/4})$ (here $c>0$ is an absolute constant), it holds that:
    $$
    \left\|\frac1N \Xi^T \Xi - \Sigma\right\|\leqslant \gamma.
    $$
    Now, for $D=d(d+1)/2$:
    $$
     \left\|\frac1N \Xi^T \Xi - \Sigma\right\|\leqslant \gamma \iff \forall v\in\R^{D}, \left|\|\frac{1}{\sqrt{N}}\Xi v\|_2^2 - v^T \Sigma v \right|\leqslant \gamma\|v\|_2^2,
    $$
    which implies, for every $v$ on the sphere $\mathbb{S}^{D-1}=\{v\in\mathbb{S}^{D}:\|v\|_2=1\}$,
    $$
    \frac1N \|\Xi v\|_2^2 \geqslant v^T \Sigma v-\gamma \Rightarrow \frac1N \inf_{v:\|v\|=1}\|\Xi v\|_2^2 \geqslant \inf_{v:\|v\|=1}v^T \Sigma v - \gamma.
    $$
    Now, using the Courant-Fischer variational characterization of the smallest singular value \cite{horn2012matrix}, we obtain
    \begin{equation}\label{eq:dddddd}
        \sigma_{\min}(\Xi) \geqslant N(\sigma_{\min}(\Sigma)-\gamma)>\frac{N}{2}\sigma_{\min}(\Sigma),
    \end{equation}
    with probability at least $1-\exp(-c'N^{1/4})$, where $c'>0$ is a positive  absolute constant smaller than $c$.

We now return to (\ref{eq:salca}), to specifically bound $\|(\Xi^T\Xi)^{-1}\|$. Let $A$ be any matrix $A$. Note that, $\|A^{-1}\|= \sigma_{\min}(A)^{-1}$. Indeed, taking the singular value decomposition $A=U\Sigma V^T$, and observing, $A^{-1}=(V^T)^{-1}\Sigma^{-1} U^{-1}$ we obtain $\|A^{-1}\| = \max_i (\sigma_i(A))^{-1} = \sigma_{\min}(A)^{-1}$. This, together with (\ref{eq:dddddd}), yields:
   \begin{equation}\label{eq:aha-bu-iki}
    \|(\Xi^T \Xi)^{-1}\| \leqslant \frac{2}{N\sigma_{\min}(\Sigma)},
      \end{equation}
    with probability at least $1-\exp(-c'N^{1/4})$.
    
We now have all ingredients to execute the bound in (\ref{eq:salca}). Combining Equations (\ref{eq:salca-1}) and (\ref{eq:aha-bu-iki}), we get:
\begin{align*}
    \mathcal{M} = (\Xi^T \Xi)^{-1}\Xi^T v\Rightarrow \|\mathcal{M}\|_2^2&\leqslant \|(\Xi^T\Xi)^{-1}\|_2^2 \cdot \|\Xi^Tv\|_2^2 \\
    &\leqslant \underbrace{\frac{4}{N^2\sigma_{\min}(\Sigma)^2}}_{\text{from (\ref{eq:aha-bu-iki})}}\cdot \underbrace{N^2d^{2K_1+2}\kappa}_{\text{from \ref{eq:salca-1} }} \\
    &=4\kappa \sigma_{\min}(\Sigma)^{-2}d^{2K_1+2}\leqslant 4C\kappa d^{2K_1+10},
\end{align*}
for some constant $C>0$. From part ${\rm (b)}$ done above, we have that $\kappa$ can be taken 
$$
32\epsilon\sigma_{\min}(W^*)^{-2}d^{4K_2+4}
$$
with probability at least
$$
   1-\exp(-CN)-(9d^{4K_1+4K_2+3})^{d^2-1}\left(\exp(-C_4 Nd^{-4K_1-4K_2-2}) + Nd\exp(-Cd^{2K_1})\right).
$$
Since $\|\mathcal{M}\|_2^2\leqslant 4C\kappa d^{2K_1+10}$ with probability at least $1-\exp(-c'N^{1/4})$, we have that
$$
\|\mathcal{M}\|_2 \leqslant C'\sqrt{\epsilon}d^{K_1+2K_2+7}\sigma_{\min}(W^*)^{-1}
$$
with probability  at least 
$$
1-\exp(-c'N^{1/4}) -(9d^{4K_1+4K_2+3})^{d^2-1}\left( \exp(-C_4 Nd^{-4K_1-4K_2-2}) +(9d^{4K_1+4K_2+3})^{d^2-1} Nd\exp(-Cd^{2K_1})\right),
$$
by the union bound. As $\|M\|_F\leqslant \|\mathcal{M}\|_2$, we have the conclusion. 

We now show the generalization ability. For any $W\in\R^{m\times d}$, using auxiliary result Theorem~\ref{thm:analytic-exp-pop-risk}{\rm (c)} we have
$$
\mathcal{L}(W)\leqslant \mu_2^2{\rm trace}(M) + \max\{\mu_4-\mu_2^2,2\mu_2^2\}{\rm trace}(M^2),
$$
where $M=W^TW-(W^*)^T W^*\in\R^{d\times d}$. Now note that ${\rm trace}(M)^2 = |\sum_{1\leqslant i\leqslant d}M_{ii}|^2\leqslant d\sum_{1\leqslant i\leqslant d}M_{ii}^2\leqslant d\|M\|_F^2$ by Cauchy-Schwarz. Furthermore ${\rm trace}(M^2)={\rm trace}(M^T M)=\|M\|_F^2$, thus yielding 
$$
\mathcal{L}(W)\leqslant \|M\|_F^2 \left(d\mu_2^2+\max\{\mu_4-\mu_2^2,2\mu_2^2\}\right)\leqslant  2d\mu_2^2\|M\|_F^2,
$$
for $d$ large. Since $\|M\|_F^2\leqslant \|\mathcal{M}\|_2^2 \leqslant C'\epsilon d^{2K_1+4K_2+14}\sigma_{\min}(W^*)^{-2}$, we conclude the proof of this part. 


\subsubsection*{Part {\rm (d)}}
Suppose the events $\mathcal{E}_1$ (defined in \eqref{eq:upper-bd-normW-belowbarrier}), $\mathcal{E}_2$ (defined in \eqref{eq:highprobdatabd}), and $\mathcal{E}_3$ (defined in \eqref{eq:event-3}), hold simultaneously; happening with probability at least
$$
1-2\exp(-CN)-\left(9d^{4K_1+4K_2+3}\right)^{d^2-1}\left(\exp\left(-C_3 Nd^{-4K_1-4K_2-2}\right)+Nde^{-Cd^{2K_1}}\right).
$$
    In particular, on this event, it holds 1) every $W$ with objective value at most $\risk{W_0}$ has Frobenius norm bounded  above by $d^{K_2+1}$, and $L=\sup\{\|\nabla^2 \risk{W}\|:\risk{W}\leqslant \riskk_0\}={\rm poly}(d)<\infty$, and 2) every rank-deficient $W$ with Frobenius norm at most  $d^{K_2+1}$ has objective value larger than $\risk{W_0}$. 
    
    We now establish $\|\nabla \risk{W}_k\|_F\to 0$ as $k\to\infty$. For this it is convenient to recall the findings from the proof of Part {\rm (b)} above: the gradient descent with a step size of $\eta<1/2L$ generates a trajectory $\{W_k\}_{k\geqslant 0}$ on which 1) $\risk{W_k}\geqslant \risk{W_{k+1}}$ for any $k\geqslant 0$; and 2) 
    $$
    \risk{W_{k+1}}-\risk{W_k}\leqslant -\eta \|\nabla \risk{W_k}\|_F^2/2.
    $$
    Note also that the objective function is lower bounded (by zero). If the gradient is non-vanishing then (by passing to an appropriate subsequence if necessary), we arrive at the conclusion that each step reduces the objective function at least by a certain amount, that is uniformly bounded away from zero. But this contradicts with the fact that the objective is lower-bounded. Thus, we obtain
    $$
    \lim_{k\to\infty}\|\nabla \risk{W_k}\|_F = 0.
    $$
    Observe now that in this event we are considering, it is the case that 1) $\risk{W_k}<\riskk_0$, 2) $\|W\|_F\leqslant d^{K_2+1}$ for any $W$ with $\risk{W}\leqslant \riskk_0$; and 3) for any $\|W\|_F\leqslant d^{K_2+1}$ with ${\rm rank}(W)<d$, $\risk{W}>\riskk_0$. Hence, we deduce $W_k\in\R^{m\times d}$ is full-rank, for all $k$. 
    
    We now establish 
    $$
    \lim_{k\to\infty}\risk{W_k}=0.
    $$
    To see this observe that the sequence $\{\risk{W_k}\}_{k\geqslant 0}$ is monotonically non-increasing, and furthermore bounded by zero from below. Hence, $\lim_{k\to\infty}\risk{W_k}\triangleq \ell$ exists, as per \cite[Theorem~3.14]{rudin1964principles}. 
    
       Since the  weights remain bounded along the trajectory, it follows that there exists a subsequence $\{W_{k_n}\}_{n\in\mathbb{N}}$  with  a limit, that is, $W_{k_n}\to W^{\infty}$ as $n\to\infty$,  where $W^\infty \in\mathbb{R}^{m\times d}$. Now, the continuity of $\nabla \risk{\cdot}$, together with the continuity of the norm $\|\cdot \|_2$, imply that $\|\nabla \risk{W^\infty}\|_F=0$. Furthermore, continuity of $\risk{\cdot}$ then implies $\risk{W^\infty}=\ell$. Now, since $W_{k_n}$'s are such that $\risk{W_{k_n}}\leqslant \riskk_0$ for all $n \in \mathbb{N}$, and $\risk{W_0}$ is stricly smaller than the rank-deficient energy barrier, by taking limits as $k \rightarrow \infty$ and using discussion above, we conclude that $W^\infty$ is full rank. Since $W^\infty$ is also a stationary point of the loss, by Theorem~\ref{thm:full-rank-empirical-global-opt}, we deduce $\risk{W^\infty}=0$, which yields $\ell=0$, as desired.


   \subsection{Proof of Theorem \ref{thm:gd-conv} }\label{sec:pf-thm:gd-conv}
   The proof follows the exact same outline, as in proof of Theorem~\ref{thm:gd-conv-empirical}. Nevertheless, we provide the whole proof for completeness.
   
   Let $\{W_t\}_{t\geqslant 0}$ be a sequence of $m\times d$ matrices corresponding to the weights along the trajectory of gradient descent, that is, $W_t\in\R^{m\times d}$ is the weight matrix at iteration $t$ of the algorithm. We first show $L<\infty$. To see this, recall Theorem \ref{thm:analytic-exp-pop-risk} ${\rm (c)}$: $\mathcal{L}(W)\geqslant \mu_2^2 \cdot{\rm trace}(A)^2$,
where ${\rm trace}(A)=\|W\|_F^2 -\|W^*\|_F^2$. In particular, this yields
$\mu_2^2(\|W\|_F^2-\|W^*\|_F^2)^2 \leqslant \mathcal{L}(W)$. Hence, for any $W$ with $\mathcal{L}(W)\leqslant \mathcal{L}(W_0)$, it holds that 
$$
\|W\|_F\leqslant \left(\frac{\sqrt{\mathcal{L}(W_0)}}{\mu_2}+\|W^*\|_F^2 \right)^{1/2}<\infty.
$$
   Namely, the (Frobenius) norm of the weights of any $W$ with $\mathcal{L}(W)\leqslant \mathcal{L}(W_0)$ remains uniformly bounded from above. This, in turn, yields that the (spectral norm of the) Hessian of the objective function remains uniformly bound from above for any such $W$, since the objective is a polynomial function of $W$, which is precisely what we denote by $L$. 
   
   We now run gradient descent with a step size of $\eta<1/2L$: a second order Taylor expansion reveals that
   $$
   \mathcal{L}(W_1) - \mathcal{L}(W_0)\leqslant -\eta \|\nabla \mathcal{L}(W_0)\|_2^2/2,
   $$
   where $\nabla \mathcal{L}(W)$ is the gradient of the population risk, evaluated at $W$. 
   
   In particular, $\mathcal{L}(W_1) \leqslant \mathcal{L}(W_0)$, and furthermore, $\|\nabla^2 \mathcal{L}(W_1)\|\leqslant L$, where $\|\nabla^2 \mathcal{L}(W)\|$ is the spectral norm of the Hessian matrix $\nabla^2 \mathcal{L}(W)$. From here, we induct on $k$: induction argument reveals we can retain a step size of $\eta<1/2L$,  and furthermore we deduce that the gradient descent trajectory $\{W_k\}_{k\geqslant 0}$ is such that: ${\rm (i)}$ $\mathcal{L}(W_k)\geqslant \mathcal{L}(W_{k+1})$, for every $k\geqslant 0$, and furthermore, ${\rm (ii)}$ it holds for every $k\geqslant 0$:
   $$
   \mathcal{L}(W_{k+1})-\mathcal{L}(W_k)\leqslant -\eta\|\nabla \mathcal{L}(W_k)\|_2^2/2.
   $$
   We now establish that $\|\nabla \mathcal{L}(W_k)\|_2\to 0$ as $k\to\infty$. Note that the objective function is lower bounded (by zero). If the gradient is non-vanishing then (by passing to a subsequence, if necessary) each step reduces the value of the objective function at least by a certain amount, that is (uniformly) bounded away from zero. But this contradicts with the fact that the objective is lower bounded.  Thus we deduce
   $$
   \lim_{k \rightarrow \infty}\|\nabla \mathcal{L}(W_k)\|_2= 0.
   $$
   Now, recall that the trajectory is such that $\mathcal{L}(W_k)\geqslant \mathcal{L}(W_{k+1})$, and that, $\|\nabla  \mathcal{L}(W_k)\|_2\to 0$ as $k\to\infty$. 
   Suppose that the initial value, $\mathcal{L}(W_0)$, is such that
   $$
   \mathcal{L}(W_0)<\min_{W\in\R^{m\times d}:{\rm rank}(W)<d} \mathcal{L}(W).
   $$
   In particular, for every $k\in\mathbb{Z}^+$,  \begin{equation}\label{eq:barrier_all_k}
   \mathcal{L}(W_k) \leqslant \mathcal{L}(W_0)<\min_{W\in\R^{m\times d}:{\rm rank}(W)<d} \mathcal{L}(W).
   \end{equation} and therefore $W_k\in\R^{m\times d}$ is full-rank, for all $k$, per Theorem \ref{thm:band-gap}. We now establish
   $$
   \lim_{k\to\infty} \mathcal{L}(W_k)=0.
   $$
   To see this, observe that the sequence $\{\mathcal{L}(W_k)\}_{k\geqslant 0}$ is monotonic (non-increasing), and furthermore, is bounded by zero from below. Hence,
   $$
   \lim_{k\to\infty}\mathcal{L}(W_k)\triangleq \ell
   $$
   exists \cite[Theorem~3.14]{rudin1964principles}. We now show $\ell=0$. 
   
   Since the  weights remain bounded along the trajectory, it follows that there exists a subsequence $\{W_{k_n}\}_{n\in\mathbb{N}}$  with  a limit, that is, $W_{k_n}\to W^{\infty}$ as $n\to\infty$,  where $W^\infty \in\mathbb{R}^{m\times d}$. Now, the continuity of $\nabla \mathcal{L}$, together with the continuity of the norm $\|\cdot \|_2$, imply that $\|\nabla \mathcal{L}(W^\infty)\|_2=0$. Furthermore, continuity of $\mathcal{L}(\cdot)$ then implies $\mathcal{L}(W^\infty)=\ell$. Now, since $W_{k_n}$'s are such that $\mathcal{L}(W_{k_n})\leqslant \mathcal{L}(W_0)$ for all $n \in \mathbb{N}$, and $\mathcal{L}(W_0)$ is stricly smaller than the rank-deficient energy barrier, by taking limits as $k \rightarrow \infty$ and using \eqref{eq:barrier_all_k}, we conclude that $W^\infty$ is full rank. Since $W^\infty$ is also a stationary point of the loss, by Theorem \ref{thm:full-rank-global-opt}, we deduce $\mathcal{L}(W^\infty)=0$, which yields $\ell=0$, as desired.
\subsection{Proof of Theorem \ref{thm:initialization-empirical} }\label{sec:pf-thm:initialization-empirical}
Let $W_0^T W_0=mI_d$, and let $\{\lambda_1,\dots,\lambda_d\}=\sigma((W^*)^T W^*-mI_d)$. In what follows below, recall the quantities from the proof of Theorem \ref{thm:initialization}{\rm (b)}: $\sigma_*\triangleq {\rm Var}((W_{ij}^*)^2-1)$, $\chi_2\triangleq \int x^2 \; d\omega(x)$, where $\omega(x)$ is the semicircle law. Fix now an arbitrary $\epsilon>0$ and a $K>0$. 

We start by definining several auxiliary events:
\begin{align*}
    \mathcal{E}_1 &\triangleq \left\{\sum_{1\leq i\leq d}\lambda_i^2 < 4(1+o(1))md^2\chi_2\right\},\\
    \mathcal{E}_2 &\triangleq \left\{\left|\sum_{1\leq i\leq d}\lambda_i\right|<\sigma_*\sqrt{md}d^{\epsilon}\right\},\\
       \mathcal{E}_3 &\triangleq\left\{\sigma_{\min}(W^*)^4\geqslant \frac{1}{16}m^2\right\},\\
          \mathcal{E}_4 &\triangleq \left\{\|X_i\|_\infty<d^K,1\leq i\leq N\right\}.
\end{align*}
Note that from the proof of Theorem \ref{thm:initialization}{\rm (b)}, we have $\mathbb{P}(\mathcal{E}_i)\geqslant 1-o_d(1)$ for $i=1,2,3$; and from union bound and sub-Gaussianity of $X$, $\mathbb{P}(\mathcal{E}_4)\geqslant 1-N\exp(-Cd^{2K})$. Thus,
\[
\mathbb{P}\left(\bigcap_{1\leq i\leq 4}\mathcal{E}_i\right)\geqslant 1-o_d(1)-N\exp\left(-Cd^{2K}\right).
\]
In what follows, suppose we condition on the event $\bigcap_{1\leq i\leq 4}\mathcal{E}_i$. Note that in this conditional universe, it is still the case that $X_i$, $1\leq i\leq N$ are i.i.d. random vectors with centered i.i.d. coordinates. Using now H\"{o}lder's inequality (Theorem \ref{thm:matrix-holder}) with $p=1,q=\infty$, $U=X_iX_i^T$ and $V=(W^*)^T W^*-mI_d$, we arrive at
\begin{align*}
|X_i^T ((W^*)^T W^*-mI_d) X_i|&=\left|\left\langle X_iX_i^T,(W^*)^T W^*-mI_d\right\rangle \right| \\
&\leqslant \|(W^*)^T W^*-mI_d\| {\rm trace}(X_iX_i^T)\\
&\leqslant 2\sqrt{md}d^{2K+1},
\end{align*}
where we use the fact that ${\rm trace}(X_iX_i^T)=\|X_i\|_2^2 \leqslant  d^{2K+1}$ (recall the conditioning on $\mathcal{E}_4$). Using Hoeffding's inequality, we have
\begin{align*}
\risk{W_0}=\frac1N \sum_{1\leq i\leq N}\left(X_i^T(W^*)^T W^* X_i - X_i^T W_0^T W_0X_i\right)^2\leqslant \frac32 \mathcal{L}(W_0),
\end{align*}
with probability at least
\[
1-\exp(-C' Nd^{-4K-3}m^{-1}),
\]
where
\[
\mathcal{L}(W_0)= \mathbb{E}\left[(X_i^T (W^*)^T W^* X_i - X_i^T W_0^T W_0 X_i)^2 \bigr\vert \|X_i\|_\infty<d^K\right].
\]
Namely, $\mathcal{L}(W_0)$ is the ``population risk" in the ``conditional universe". 

Next, in this conditional space, using Theorem \ref{thm:analytic-exp-pop-risk}(c), we arrive at
$$
\mathcal{L}(W_0)\leqslant \mu_2(K_1)^2\left|\sum_{1\leq i\leq d}\lambda_i\right|^2 +\max\{\mu_4(K_1)-\mu_2(K_1)^2,2\mu_2(K_1)^2\}\left(\sum_{1\leq i\leq d}\lambda_i^2\right).
 $$
 Finally, carrying out the exact same analysis as in the end of the proof of Theorem \ref{thm:initialization}, we deduce  
 $$
 \risk{W_0}<\frac12 C_5 \sigma_{\min}(W^*)^4,
 $$
 provided $m>Cd^2$ for a large enough constant $C$, namely provided that the network is sufficiently overparametrized.

\subsection{Proof of Theorem \ref{thm:initialization}}\label{sec:pf-initialization} 
\subsubsection*{Part $({\rm a})$} 
Let $t=\sqrt{d}$. Then, using Theorem \ref{thm:tall-matrix-spectra-concentrate}, it holds that with probability $1-2\exp(-d/2)$:
\begin{alignat*}{3}
\sqrt{m}-2\sqrt{d}&\leqslant \sigma_{\min}(W^*)& & \leqslant \sigma_{\max}(W^*)&&\leqslant \sqrt{m}+2\sqrt{d}  \\ \Rightarrow m+4d-4\sqrt{md}&\leqslant \lambda_{\min}((W^*)^T W^*)& & \leqslant \lambda_{max}((W^*)^T W^*)&&\leqslant m+4d+4\sqrt{md}.
\end{alignat*}
Recall that $\sigma(A)$ denotes the spectrum of $A$, i.e., $\sigma(A)=\{\lambda:\lambda\text{  is an eigenvalue of } A\}$. We claim then the spectrum of $\gamma I-A$ is $\gamma-\sigma(A)$. To see this, simply note the following line of reasoning:
$$
\gamma-\lambda\in\sigma(\gamma I-A)\iff {\rm det}((\gamma -\lambda)I -  (\gamma I-A))= 0 \Leftrightarrow {\rm det}(\lambda I -A)= 0 \Leftrightarrow \lambda \in \sigma(A).
$$
Now, let $W_0\in\R^{m\times d}$ be such that $W_0^T W_0 = \gamma I$ with $\gamma=m+4d$. In particular, if $\lambda_1\leqslant \cdots\leqslant \lambda_d$ are the eigenvalues of $\gamma I-(W^*)^T W^*$ with $\gamma=m+4d$; then, it holds that:
$$
-4\sqrt{md}\leqslant \lambda_1\leqslant \cdots\leqslant \lambda_d \leqslant 4\sqrt{md}.
$$
Now, recall by Theorem \ref{thm:analytic-exp-pop-risk} ${\rm (c)}$ that, 
$$
\mathcal{L}(W_0) \leqslant \mu_2^2 \left(\sum_{i=1}^d \lambda_i\right)^2 + \max\left\{{\rm Var}(X_i^2),2\E{X_i^2}^2 \right\} \left(\sum_{i=1}^d \lambda_i^2\right),$$
where $\sigma(W_0^T W_0-(W^*)^T W^*)  = \{\lambda_1,\dots,\lambda_d\}$. For the second term, we immediately have
$\sum_{i=1}^d \lambda_i^2 \leqslant 16md^2$. 

For the first term, note first that, if $\lambda_1'\leqslant\cdots\leqslant \lambda_d'$ are the eigenvalues of $(W^*)^T W^*$, then
$$
\sum_{k=1}^d\lambda_k'={\rm trace}((W^*)^T W^*)=\sum_{i=1}^m \sum_{j=1}^d (W^*_{ij})^2\Rightarrow \sum_{k=1}^d (\lambda_k'-m)=\sum_{i=1}^m\sum_{j=1}^d ((W_{ij}^*)^2 -1),
$$
where $W_{ij}^*\distr N(0,1)$ i.i.d.. Note also that, $(W_{ij}^*)^2 -1 $ is a centered random variable, and has sub-exponential tail, see \cite[Lemma~5.14]{vershynin2010introduction}. Now, letting  $Z_{ij}=(W_{ij}^*)^2-1$, and applying the Bernstein-type inequality \cite[Proposition~5.16]{vershynin2010introduction}, we have that for some absolute constants $K,c>0$, it holds:
$$
\mathbb{P}\left(\left|\sum_{i=1}^m\sum_{j=1}^ d Z_{ij}\right|>d\sqrt{m}\right)\leqslant 2\exp\left(-c\min\left(\frac{d}{K^2},\frac{d\sqrt{m}}{K}\right)\right)\leqslant 2\exp(-cd/K^2) = \exp(-\Omega(d)),
$$
for $m$ sufficiently large. In particular, with probability at least $1-\exp(-\Omega(d))$, it therefore holds that,
$$
\left|\sum_{k=1}^d(\lambda_k' - m)\right|\leqslant d\sqrt{m}.
$$
Finally, using triangle inequality,
$$
\left|\sum_{k=1}^d \lambda_k\right| = 
\left|\sum_{k=1}^d (\lambda_k'-(m+4d))\right|\leqslant 
\left|\sum_{k=1}^d (\lambda_k'-m)\right|+4d^2\leqslant d\sqrt{m}+4d^2,
$$
with probability $1-\exp(-\Omega(d))$. After squaring, we obtain that 
$\left(\sum_{i=1}^d \lambda_i\right)^2 \leqslant 16d^4+8d^3\sqrt{m}+ d^2m$. In particular, we get:
\begin{align*}
\mathcal{L}(W_0) &\leqslant \mu_2^2 \left(\sum_{i=1}^d \lambda_i\right)^2 + \max\left\{{\rm Var}(X_i^2),2\E{X_i^2}^2 \right\} \left(\sum_{i=1}^d \lambda_i^2\right) \\
&\leqslant  \mu_2^2 (16d^4+8d^3\sqrt{m}+ md^2) + \max\left\{{\rm Var}(X_i^2),2\E{X_i^2}^2 \right\}16md^2.
\end{align*}
Using now the overparametrization $m>Cd^2$, we further have:
\begin{align*}
\E{X_i^2}^2 (16d^4+8d^3\sqrt{m}+ md^2) + \max\left\{{\rm Var}(X_i^2),2\E{X_i^2}^2 \right\}16md^2\leqslant \mathcal{C'}(C) m^2,
\end{align*} 
where
$$
\mathcal{C'}(C) = \E{X_i^2}^2\left(\frac{16}{C^2} + \frac{8}{C^{3/2}} +\frac1C\right) +\frac{16}{C} \max\left\{{\rm Var}(X_i^2),2\E{X_i^2}^2\right\}.
$$
Note that for the constant $\mathcal{C'}(C)$, 
$$
\mathcal{C}'(C) \to 0\quad\text{as}\quad C\to+\infty.
$$
Next, observe that, $\sqrt{m}-2\sqrt{d}\geqslant \frac12\sqrt{m}$ for $m$ large (in the regime $m>Cd^2$, with $C$ large enough). Thus, using what we have established in Theorem \ref{thm:band-gap}, we  arrive at:
\begin{align*}
\min_{W\in \R^{m\times d}:{\rm rank}(W)<d}\mathcal{L}(W)&>\min\left\{{\rm Var}(X_i^2),2\E{X_i^2}^2\right\}\sigma_{\min}(W^*)^4 \\
&\geqslant \min\left\{{\rm Var}(X_i^2),2\E{X_i^2}^2\right\}(\sqrt{m}-2\sqrt{d})^4 \\
&\geqslant \frac{1}{16}\min\left\{{\rm Var}(X_i^2),2\E{X_i^2}^2\right\} m^2.
\end{align*} 
Finally, observe also that if ${\rm Var}(X_i^2)>0$, then $\E{X_i^2}>0$ as well: indeed observe that if $\E{X_i^2}=0$, then $X_i = 0$ almost surely, for which ${\rm Var}(X_i^2)=0$. In particular, $\min\left\{{\rm Var}(X_i^2),2\E{X_i^2}^2\right\}>0$. Equipped with this, we then observe that provided:
$$
\frac{1}{16}\min\left\{{\rm Var}(X_i^2),2\E{X_i^2}^2\right\} > \mathcal{C'}(C) = \E{X_i^2}^2\left(\frac{16}{C^2} + \frac{8}{C^{3/2}} +\frac1C\right) +\frac{16}{C} \max\left\{{\rm Var}(X_i^2),2\E{X_i^2}^2\right\},
$$
that is, provided $C>0$ is sufficiently large, we are done.

\subsubsection*{Part $({\rm b})$} 
Note that, the result of Bai and Yin \cite{bai1988convergence} asserts that if $\{\mu_1,\dots,\mu_d\}$ are the eigenvalues of 
$$
\mathcal{A}\triangleq \frac{1}{2\sqrt{md}}((W^*)^T W^* - mI_d),
$$
and if we define the empirical measure
$$
F^{\mathcal{A}}(x) = \frac1d \left|\{i: \mu_i \leqslant x\}\right|
$$
then in the regime $d\to+\infty$, $d/m\to 0$, it holds that:
$$
F^{\mathcal{A}}(x) \to \omega(x),
$$
almost surely, where $\omega(x)$ is the semicircle law; and moreover
$$
\frac{1}{d}\sum_{i=1}^d \mu_i^2 \to \int x^2\;d\omega(x)\triangleq \chi_2 
$$
namely, $\chi_2$ is respectively the second moment under semicircle law, whp. Now, define the same quantities as in proof of  part ${\rm (a)}$, where this time $W_0^T W_0 = mI_d$, and $\{\lambda_1,\dots,\lambda_d\}=\sigma((W^*)^T W^* - mI_d)$. In particular, we still retain the inequality per Theorem \ref{thm:analytic-exp-pop-risk} ${\rm (c)}$: 
$$
\mathcal{L}(W_0) \leqslant \mu_2^2 \left(\sum_{i=1}^d \lambda_i\right)^2 + \max\left\{{\rm Var}(X_i^2),2\E{X_i^2}^2 \right\} \left(\sum_{i=1}^d \lambda_i^2\right).$$
Note that $\lambda_i=2\sqrt{md}\mu_i$. Hence, we obtain
$$
\sum_{i=1}^d \lambda_i^2 <(4+o(1))md^2 \chi_2
$$
whp. We now control $\sum_{i=1}^d \lambda_i$ using central limit theorem (CLT). Observe that,
$$
\sum_{i=1}^d \lambda_i = {\rm trace}((W^*)^T W^* - mI_d) = \sum_{i=1}^m \sum_{j=1}^d ((W_{ij}^*)^2 -1).
$$
Now, note that
$$
\sigma_*^2\triangleq {\rm Var}((W_{ij}^*)^2-1)  = {\rm Var}((W_{ij}^*)^2)<\E{(W_{ij}^*)^4}<\infty.
$$
We now use CLT, as $d\to\infty$ and $m/d\to \infty$. To that end, let $1/2>\epsilon>0$ be fixed. Observe now that, for any arbitrary $M>0$, and sufficiently large $d$, 
$$
\left\{-1\leqslant \frac{1}{\sigma_* \sqrt{md}d^{\epsilon}} \sum_{i=1}^m \sum_{j=1}^d ((W_{ij}^*)^2 -1)\leqslant 1\right\}\supset \left\{-M\leqslant \frac{1}{\sigma_* \sqrt{md}} \sum_{i=1}^m \sum_{j=1}^d ((W_{ij}^*)^2 -1)\leqslant M\right\}.
$$
In particular, using central limit theorem, we deduce
$$
\liminf_{d\to\infty}\mathbb{P}\left(-1\leqslant \frac{1}{\sigma_* \sqrt{md}d^{\epsilon}} \sum_{i=1}^m \sum_{j=1}^d ((W_{ij}^*)^2 -1)\leqslant 1\right)\geqslant \mathbb{P}(Z\in [-M,M]),
$$
where $Z$ is a standard normal random variable. Now since $M>0$ is arbitrary, we have, by sending $M\to+\infty$, we obtain
$$
\liminf_{d\to\infty}\mathbb{P}\left(-1\leqslant \frac{1}{\sigma_* \sqrt{md}d^{\epsilon}} \sum_{i=1}^m \sum_{j=1}^d ((W_{ij}^*)^2 -1)\leqslant 1\right)\geqslant 1,
$$
and we then conclude
$$
\lim_{d\to\infty}\mathbb{P}\left(-1\leqslant \frac{1}{\sigma_* \sqrt{md}d^{\epsilon}} \sum_{i=1}^m \sum_{j=1}^d ((W_{ij}^*)^2 -1)\leqslant 1\right)=1.
$$
Hence,
$$
\left|\sum_{i=1}^d \lambda_i\right|\leqslant \sigma_*\sqrt{md}d^{\epsilon},
$$
with probability $1-o_d(1)$, for $d$ sufficiently large. 

Moreover, 
$$
\sigma_{\min}(W^*)^4 \geqslant \frac{1}{16}m^2,
$$
for $m$ large, using yet another result of Bai and Yin, see Theorem \ref{thm:baiyinvershy}. From here, carrying the exact same analysis as in part ${\rm (a)}$  we obtain provided $m>Cd^2$ for some large constant $C>0$, and $d$ sufficiently large the following holds with probability $1-o_d(1)$:
$$
\mathcal{L}(W_0)<\min_{W\in\R^{m\times d}:{\rm rank}(W)<d} \mathcal{L}(W),
$$
where $W_0$ is prescribed such that $W_0^T W_0 = mI_d$. 
   
\subsection{Proof of Theorem \ref{thm:geo-condition}}\label{sec:proof-of-geo-cond}
\begin{itemize}
     \item[(a)] Let ${\rm span}(X_iX_i^T:i\in[N])=\mathcal{S}$, the set of all $d\times d$ symmetric matrices, and let $M\in\mathcal{S}$ be such that for any $i$, $X_i^T M X_i=0$. We will establish $M=0$. Let $1\leqslant k,\ell \leqslant d$ be two fixed indices. To that end, let $\theta_i^{(k,\ell)}\in\R$ be such that, $\sum_{i=1}^N \theta_i^{(k,\ell)}X_iX_i^T =e_ke_\ell^T + e_\ell e_k^T$, where the column vectors $e_k,e_\ell \in\R^d$ are respectively the $k^{th}$ and $\ell^{th}$ elements of the standard basis  for $\R^d$. Such $\theta_i^{(k,\ell)}$ indeed exist, due to the spanning property. Observe that $2M_{k,\ell}=e_k^T Me_\ell +e_\ell ^T Me_k ={\rm tr}(e_k^T Me_\ell+e_\ell^T Me_k)$. Now, using the fact that ${\rm tr}(ABC)={\rm tr}(BCA)={\rm tr}(CAB)$ for every matrices $A,B,C$ (with matching dimensions), we have:
    \[
    2M_{k,\ell}= {\rm tr}(Me_\ell e_k^T +Me_ke_\ell^T)={\rm tr}\left(\sum_{i=1}^N \theta_i^{(k,\ell)}MX_iX_i^T\right) = \sum_{i=1}^N \theta_i^{(k,\ell)}{\rm tr}(X_i^T MX_i)=0,
    \]
for every $k,\ell\in[d]$. Finally, if $W$ is such that $\risk{W}=0$, then $X_i^T M X_i=0$ for any $i$, where $M=(W^*)^T W^* - W^T W$. Hence, provided that the geometric condition holds, we have $M=0$, that is, $W^T W=(W^*)^T W^*$. From here, the final conclusion follows per Theorem \ref{thm:auxiliary-2}. Since $W^T W= (W^*)^ TW^*$, $W$ clearly has zero generalization error, i.e. $\mathcal{L}(W)=0$.
\item[(b)] Our goal is to construct a $W\in\mathbb{R}^{m\times d}$ with $f(W^*;X_i)=f(W;X_i)$, for every $i\in[N]$, whereas $W^TW\neq (W^*)^T W^*$. Consider the inner product $\ip{A}{B}={\rm trace}(AB)$, in the space of all symmetric $d\times d$ matrices. Find $0\neq M\in\R^{d\times d}$ 
 a symmetric matrix, such that, $M\in {\rm span}^{\perp}(X_iX_i^T:i\in[N])$, that is, $X_i^T M X_i=0$ for every $i\in[N]$. 
We can find such $M$ satisfying  $\|M\|_2=1$.  Consider the linear matrix function $M(\delta)=(W^*)^T W^*+\delta M$. Note that, $M(\delta)$ is symmetric for every $\delta$. We claim that under the hypothesis of the theorem, there exists a $\delta_0>0$ such that $M(\delta)$ is positive semidefinite for every $\delta\in[0,\delta_0]$, and that there exists $W_\delta\in\R^{m\times d}$ with $W_\delta^T W_\delta= M(\delta)$, for all $\delta\in[0,\delta_0]$. Observe that, since ${\rm rank}(W^*)=d$, then $(W^*)^T W^* \in \R^{d\times d}$ with ${\rm rank}((W^*)^T W^*)=d$.  Therefore, the eigenvalues $\lambda_1^*,\dots,\lambda_d^*$ of $(W^*)^T W^*$ 
are all positive.  In particular $\{\lambda_i^*:i\in[d]\}\subset [\delta_1,\infty)$, with $\delta_1 =\sigma_{\min}(W^*)^2$. 
Now, let $\mu_1(\delta),\dots,\mu_d(\delta)$ be the eigenvalues of $M(\delta)$. Using Weyl's inequality \cite{horn2012matrix}, we have $|\mu_i(\delta)-\lambda_i^*|\leqslant \delta\|M\|_2=\delta$, for every $i$. In particular, taking $\delta\leqslant  \delta_1$, we deduce for every $i\in[d]$, it holds that $\mu_i(\delta)\geqslant \lambda_i^*-\delta_1\geqslant 0$,   that is, $\{\mu_i(\delta):i\in[d]\}\subset [0,\infty)$. In  particular, we also have $M(\delta)$ is symmetric, and thus, it is PSD. Thus, there exists a $\overline{W_\delta}\in \R^{d\times d}$ such that $\overline{W_\delta}^T \overline{W_\delta}=M(\delta)$. Now, using the same idea as in the proof of Theorem \ref{thm:band-gap} part $(c)$, we then deduce that for any $\widehat{m}\geqslant d$, there exists a matrix $W_\delta\in\R^{\widehat{m}\times d}$ such that $W_{\delta}^T W_{\delta}=\overline{W_\delta}^T \overline{W_\delta}=M(\delta)$.  In particular, for this $W_\delta$, if $f(W_\delta,X)$ is the function computed by the neural network with weight matrix $W_\delta\in\mathbb{R}^{\widehat{m}\times d}$, then on the training data $(X_i:i\in[N])$, $f(W_\delta;X_i)=X_i^T W_\delta^T W_\delta X_i = X_i^T (W^*)^T W^* X_i = f(W^*;X_i)$, 
since $X_i^T M X_i=0$ for all $i\in[N]$. At the same time  $W_\delta^T W_\delta - (W^*)^T W^* = \delta M\neq 0$, since $\delta\neq 0$ and $M\neq 0$,
and therefore $W_\delta^T W_\delta \neq (W^*)^T W^*$.

Finally, to show $\mathcal{L}(W_\delta)>0$, we argue as follows. Suppose $\mathcal{L}(W_\delta)=0$. Then, by Theorem \ref{thm:auxiliary}, it follows that $\psi(X)=X^T A X=0$ identically, where $A=W_\delta^T W_\delta -(W^*)^T W^*$. Now, letting $\xi_1,\dots,\xi_d$ to be the eigenvectors of $A$ (with corresponding eigenvalues $\lambda_1,\dots,\lambda_d$), we obtain $\xi_i^T A\xi_i = \lambda_i \xi_i^T \xi_i = \lambda_i \|\xi_i\|_2^2 = 0$, we namely obtain $\lambda_i=0$ for every $i\in[d]$. Finally, since $A$ is symmetric, and hence admits a diagonalization of form $A=\mathcal{Q}\Lambda \mathcal{Q}$ with diagonal entries of $\Lambda$ being zero, we deduce $A$ is identically zero, which contradicts with the fact that $A=\delta M$, which is a non-zero matrix.

\end{itemize}

\subsection{Proof of Corollary \ref{coro-2}}\label{sec:proof-of-coro-2}
The proof relies on the following observation: given any pair $(a^*,W^*)\in\R_+^m\times \R^{m\times d}$, construct a matrix $\widehat{W}^*\in\R^{m\times d}$ whose $j^{th}$ row is $\widehat{W}_j^* = \sqrt{a_j^*}W_j^*\in\R^d$. Define $\widehat{W}\in\R^{\widehat{m} \times d}$ similarly as the matrix whose $j^{th}$ row is $\widehat{W}_j = \sqrt{a_j}W_j\in\R^d$. Now, let $e^{(m)} =(1,1,\dots,1)^T\in\R^m$  and $e^{(\widehat{m})}=(1,1,\dots,1)^T\in \R^{\widehat{m} }$ be the vector of all ones. Then note that, 
\[
\widehat{f}(a^*,W^*,X) = \widehat{f}(e^{(m)} ,\widehat{W}^*,X) = f(\widehat{W}^*,X) \quad\text{and}\quad \widehat{f}(a,W,X) = \widehat{f}(e^{(\widehat{m})},\widehat{W},X) = f(\widehat{W},X),
\] 
where $f(\widehat{W};X)$ is the same quantity as in Theorem \ref{thm:geo-condition}. Applying Theorem \ref{thm:geo-condition} then establishes both parts. 

\subsection{Proof of Theorem \ref{thm:random-data-geo-cond} }\label{sec:proof-of-random-data-geo}
Recall that, $\mathcal{S}=\{M\in\R^{d\times d}:M^T=M\}$. Note that, this space has dimension $\binom{d}{2}+d$: for any $1\leqslant k\leqslant \ell\leqslant d$, it is easy to see that the matrices $e_ke_\ell^T + e_\ell e_k^T$ are linearly independent; and there are precisely $\binom{d}{2}+d$ such matrices. With this in mind, the statement of part $(b)$ is immediate. 

We now prove the part $(a)$ of the theorem. For any $X_i$, let $X_i(j)$ be the $j^{th}$ coordinate of $X_i$, with $j\in[d]$; and let $\mathcal{Y}_i$ be a $d(d+1)/2-$dimensional vector, obtained by retaining $X_i(1)^2,\dots,X_i(d)^2$; and the products, $X_i(k)X_i(\ell)$ with $1\leqslant k<\ell\leqslant d$. Now, let $\mathcal{X}$ be an $n\times d(d+1)/2$ matrix, whose rows are $\mathcal{Y}_1,\dots,\mathcal{Y}_n$. Our goal is to establish,
$$
\mathbb{P}[{\rm det}(\mathcal{X})=0]=0,
$$
when $n=d(d+1)/2$, where the probability is taken with respect to the randomness in $X_1,\dots,X_n$ (in particular, this yields  for $n\geqslant d(d+1)/2$, $\mathbb{P}({\rm rank}(\mathcal{X})=d(d+1)/2)$, almost surely). Now, recalling Theorem \ref{thm:auxiliary}, it then suffices to show that ${\rm det}(\mathcal{X})$ is not identically zero, when viewed as a polynomial in $X_i(j)$ with $i\in[N]$, $j\in[d]$.

We now prove part (b) by providing a deterministic construction (of the matrix $\mathcal{X})$ 
under which ${\rm det}(\mathcal{X})\neq 0$. Let $p_1<\cdots<p_d$ be distinct prime numbers. For every $1\leqslant t\leqslant N$, set:
$$
X_t = (p_1^{t-1},\dots,p_d^{t-1})^T \in \R^d.
$$
In particular, $X_1=(1,1,\dots,1)^T\in\R^d$, which then implies $\mathcal{Y}_1$ is a vector of all ones. Now, we study $\mathcal{Y}_2$. The entries of $\mathcal{Y}_2$, called $z_1,\dots,z_{d(d+1)/2}$, are of form $p_i^2$ with $i\in[d]$; or $p_i p_j$, where $1\leqslant i<j\leqslant d$. By the fundamental theorem of arithmetic, we have $p_i p_j= p_kp_\ell \Rightarrow \{p_i,p_j\}=\{p_k,p_\ell\}$; and therefore, $z_1,\dots,z_{d(d+1)/2}$ are pairwise distinct. With this construction, the matrix $\mathcal{X}$ is a Vandermonde matrix with determinant:
$$
\prod_{1\leqslant k<\ell\leqslant d(d+1)/2}(z_k-z_\ell).
$$
Since $z_k\neq z_\ell$ for every $k\neq \ell$ (from the construction on $\mathcal{Y}_2$, which, in turn, is constructed from $X_2$), this determinant is non-zero, proving the  claim.

\subsection{Proof of Theorem \ref{thm:main}}\label{sec:proof-of-thm-mainn}
\begin{itemize}
    \item[(a)] Note that, if $N\geqslant N^*$, then combining parts $(a)$ of Theorems \ref{thm:geo-condition} and \ref{thm:random-data-geo-cond}, we have that with probability one, ${\rm span}(X_iX_i^T :i\in[N])=\mathcal{S}$, which, together with $\risk{W}=0$, imply that,
$$
\mathbb{P}(E\neq \varnothing)=0,
$$
where $E=\{W\in\R^{m\times d}:W^T W \neq (W^*)^T W^*; \risk{W}=0\}$, from which the desired conclusion follows.
    \item[(b)] Assume $W$ is taken as in proof of Theorem \ref{thm:geo-condition} (b), that is,
    $$
   A= (W^*)^T W^* - W^T W=\delta M\quad\text{where}\quad \delta =\sigma_{\min}(W^*)^2\quad\text{and}\quad \|M\|=1,
    $$
    with $M^T=M$. Let $\{\lambda_1,\dots,\lambda_d\}$ be the spectrum of the matrix $\delta M$. Using now Theorem \ref{thm:analytic-exp-pop-risk} $({\rm c})$, we have the lower bound
    \begin{align*}
  	\mathcal{L}(W) &\geqslant  \E{X_i(j)^2}^2 {\rm trace}(A)^2 + \min\left\{{\rm Var}(X_i(j)^2),2\E{X_i(j)^2}^2\right\}\cdot {\rm  trace}(A^2) \\
	&\geqslant \min\left\{{\rm Var}(X_i(j)^2),2\E{X_i(j)^2}^2\right\}\left(\sum_{i=1}^d \lambda_i^2\right)\\
	&\geqslant \min\left\{{\rm Var}(X_i(j)^2),2\E{X_i(j)^2}^2\right\}\lambda_{\max}(\delta M)^2,
    \end{align*}
   since $ {\rm trace}(A^2)=\sum_{i=1}^d \lambda_i^2$. Finally, since $\lambda_{\max}(\delta M)^2 =\delta^2=\sigma_{\min}(W^*)^4$ (as the spectral norm of  $M$ is one), we arrive at the desired conclusion.
    \end{itemize}

\section*{Acknowledgement}  
The authors would like to thank Orestis Plevrakis for providing useful feedback on the initial version of this paper.

\bibliographystyle{amsalpha}
\bibliography{bibliography}

\end{document}